\ifdefined\pdfoutput
\pdfoutput=1
\fi

\documentclass[conference]{IEEEtran}
\IEEEoverridecommandlockouts

\usepackage{enumitem}

\usepackage[utf8]{inputenc} 
\usepackage[T1]{fontenc}    
\usepackage[scaled=0.85]{beramono}
\usepackage{hyperref}       
\usepackage{algorithm}
\usepackage{algorithmic}
\usepackage{tikz}
\usepackage{xcolor}

\usetikzlibrary{shapes.geometric}
\usepackage{array}
\usepackage[utf8]{inputenc} 
\usepackage[T1]{fontenc}    
\usepackage{url}            
\usepackage{booktabs}       
\usepackage{amsfonts}       
\usepackage{nicefrac}       
\usepackage{microtype}      
\usepackage{xcolor}         
\usepackage{graphicx}
\usepackage{amsmath}
\usepackage{amsthm}
\usepackage{bm}
\usepackage{algorithm}
\usepackage{algorithmic}
\usepackage{subfigure}
\usepackage{array}
\usepackage{xcolor,pifont}
\usepackage{tikz}
\usetikzlibrary{shapes.geometric}
\usetikzlibrary{positioning}
\usepackage{booktabs}
\usepackage{siunitx}
\usepackage{enumitem}


\newcommand{\nop}[1]{}

\renewcommand\thetable{\arabic{table}}
\newcommand{\partitle}[1]{\medskip \noindent \textbf{#1.}}
\newcolumntype{Y}{>{\centering\arraybackslash}X}
\usepackage{array}
\newcolumntype{C}[1]{>{\centering\arraybackslash}p{#1}}

\definecolor{cycle2}{RGB}{106, 191, 0}
\definecolor{cycle3}{RGB}{191, 0, 0}
\definecolor{amber}{rgb}{1.0, 0.75, 0.0}

\newcommand{\cmark}{\textcolor{cycle2}{\ding{52}}} %
\newcommand{\xmark}{\textcolor{cycle3}{\ding{56}}}
\makeatletter
\renewcommand\thetable{\@arabic\c@table}
\makeatother
\usepackage{booktabs}
\usepackage{tabularx}
\usepackage{makecell}
\newtheorem{theorem}{Theorem}[section]

\newtheorem{definition}[theorem]{Definition}

\newtheorem{remark}[theorem]{Remark}

\usepackage{amsmath,amsfonts,bm}

\def\eqref#1{equation~\ref{#1}}

\def\1{\bm{1}}

\def\re{{\textnormal{e}}}

\def\rve{{\mathbf{e}}}

\def\rvx{{\mathbf{x}}}

\def\vd{{\bm{d}}}

\def\vx{{\bm{x}}}

\DeclareMathAlphabet{\mathsfit}{\encodingdefault}{\sfdefault}{m}{sl}
\SetMathAlphabet{\mathsfit}{bold}{\encodingdefault}{\sfdefault}{bx}{n}

\def\sP{{\mathbb{P}}}

\def\sV{{\mathbb{V}}}

\begin{document}

\title{Feature Attribution in Directed Acyclic Graphs Using\\ Edge Intervention}

\author{
\IEEEauthorblockN{
Qiheng Sun\textsuperscript{1,2},
Junxu Liu\textsuperscript{1},
Xiaokai Mao\textsuperscript{2},
Haocheng Xia\textsuperscript{3},
Jinfei Liu\textsuperscript{2},
Kui Ren\textsuperscript{2},
Haibo Hu\textsuperscript{1}
}
\IEEEauthorblockA{
\textsuperscript{1}The Hong Kong Polytechnic University, Hong Kong, China\\
\textsuperscript{2}Zhejiang University, Hangzhou, China\\
\textsuperscript{3}University of Illinois Urbana-Champaign, Urbana, USA\\
}
\IEEEauthorblockA{
\{qiheng.sun, junxu.liu, haibo.hu\}@polyu.edu.hk\\
\{xiaokaimao, jinfeiliu, kuiren\}@zju.edu.cn, hxia7@illinois.edu
}
}

\maketitle

\begin{abstract}
Shapley value-based feature attribution methods face challenges in scenarios involving complex feature interactions and causal relationships, even when a causal structure is provided. 
Existing methods typically adopt a node-centric view, attributing importance solely to individual features.
Consequently, they often fail to simultaneously capture the externality and exogenous influence of features, leading to unreasonable interpretations. To overcome these limitations, we propose a novel feature attribution method called DAG-SHAP, which is based on edge intervention. DAG-SHAP treats each feature edge as an individual attribution object, ensuring that both externality and exogenous contributions of features are appropriately captured. Additionally, we introduce an approximation method for efficiently computing DAG-SHAP. Extensive experiments on both real and synthetic datasets validate the effectiveness of DAG-SHAP. Our code is available at \url{https://github.com/ZJU-DIVER/DAG-SHAP}.
\end{abstract}


\section{Introduction}
The growing complexity of machine learning (ML) models in real-world applications has fueled intense interest in interpretability of decision-making processes~\cite{DBLP:conf/icml/0005AW25, DBLP:conf/icml/KasmiBFP25, pmlr-v235-chen24d,pmlr-v235-machiraju24a, dominici2025causal}. Feature attribution, which aims to quantify the marginal contribution of each input feature with real-valued scores, is recognized as a leading approach for interpreting predictive outcomes in terms of their underlying explanatory variables. Within the broad landscape of attribution strategies~\cite{sundararajan2017axiomatic}, Shapley value-based methods have been extensively studied due to their rigorous foundation in cooperative game theory, ensuring fairness in contribution allocation while upholding desirable properties such as efficiency, symmetry, redundancy, and additivity~\cite{chen2022explaining, winter2002shapley}. Applications of these methods in domains such as healthcare and finance demonstrate their utility in uncovering the most influential features behind predictive outcomes~\cite{mohanty2024leveraging, kutlu2024machine,bussmann2021explainable,gramegna2021shap}.


Although widely investigated, Shapley value-based attribution methods remain limited in their ability to capture feature interactions that accurately reflect intricate data dependencies. As one of the earliest attempts, the off-manifold Shapley value~\cite{scott2017unified} relies on the assumption of feature independence, and therefore fails to capture the complex dependencies present in practical scenarios.
The on-manifold Shapley value~\cite{scott2017unified,sundararajan2020many} mitigates this limitation by imputing excluded features with conditional expectations derived from their correlations with included features when assessing subset utility. However, reliance on correlations alone may misalign with the underlying data generation process and lead to causal reversal~\cite{janzing2020feature,jung2022measuring}, particularly when causal dependencies exist among features.
To address this, studies such as asymmetric Shapley values~\cite{frye2020asymmetric} and causal Shapley values~\cite{DBLP:conf/nips/HeskesSBC20} further extend the framework by incorporating sequential and causal relationships, respectively.
In addition, recent advances characterize causal dependencies with a directed acyclic graph (DAG), where vertices represent features and edges capture direct causal effects. This makes it important to reason about edges and paths in the graph~\cite{DBLP:journals/pvldb/YanCLN14,DBLP:journals/pvldb/GuoXYYKZJ24}.
Shapley Flow~\cite{DBLP:conf/aistats/WangWL21} attributes contributions to entire pathways between source and target vertices rather than to individual causal transmissions along edges. Recursive Shapley value~\cite{DBLP:conf/icml/SingalMN21} attributes contributions following a top-down principle by first attributing to ``source'' vertices and then flowing them down the DAG. 


\begin{figure}[h]
\centering
    \begin{tikzpicture}[node distance=2cm, auto]

  \node (X1) [circle, draw, thick] {$X_1$};
  \node (X2) [circle, draw, thick, right of=X1] {$X_2$};
  \node (Y)  [circle, draw, thick, right of=X2] {$Y$};
  
  \draw[->, thick] (X1) to node[midway, above] {$\rve_1$} (X2);
  \draw[->,thick] (X2) to node[midway, above] {$\rve_3$}(Y);
  \draw[->, thick, bend left=45] (X1) to node[midway, above] {$\rve_2$} (Y);
  
  \node [right of=Y, node distance=2cm, yshift = 0.4cm] (formulas) 
    {$
    \begin{aligned}
    X_1 &= \rvx_1 \\
    X_2 &= X_1 + \rvx_2 \\
    Y_{\,\,\,} &=  f(X_1 , X_2)
    \end{aligned}
    $};
\end{tikzpicture}

\caption{An illustration of DAG with feature set $\{X_1,X_2\}$ and the model output $Y=f(X_1,X_2)$. Each directed edge indicates the direct causal effect, and $\rvx_1,\rvx_2$ represent the exogenous variables of $X_1,X_2$, respectively.
}

\label{fig:toy_example}
\end{figure}
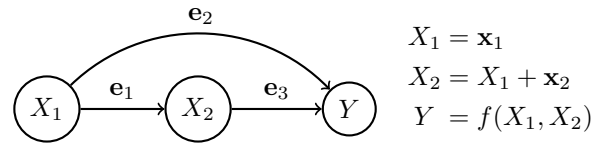

Despite the aforementioned significant efforts to enhance feature interactions in the attribution process, \textit{externality} and \textit{exogeneity} remain inadequately addressed in existing approaches.
Attribution methods based on feature vertices adopting asymmetry samplings, such as asymmetry Shapley value~\cite{frye2020asymmetric} and asymmetry causal Shapley value~\cite{DBLP:conf/nips/HeskesSBC20}, fail to make attribution results satisfy externality. When using the asymmetrical ordering of feature attribution methods (the parent vertex must appear before the child vertex), the marginal contribution of the parent vertex cannot receive externality gains from cooperating with the child vertex. This approach clearly violates the fundamental idea that each player should, to the greatest extent possible, obtain marginal contributions from the cooperation of others~\cite{roth1988shapley}. As illustrated in Figure~\ref{fig:toy_example}, the direct influence $\rve_2: X_1 \rightarrow Y$ cannot receive a marginal contribution from cooperating with $X_2$ which is $\rve_3: X_2 \rightarrow Y$,  even though there is no causal ordering between these two incoming edges into \(Y\). In addition, the attribution methods focusing on attributing contributions at each graph cut like Shapley Flow~\cite{DBLP:conf/aistats/WangWL21} and Recursive Shapley value~\cite{DBLP:conf/icml/SingalMN21}, fail to recognize exogenous contributions of features. Exogenous contribution refers to the part of the contribution in each feature that is not influenced by other features in the explained input. Shapley Flow and Recursive Shapley value assume that only feature vertices without incoming edges have exogenous contributions. Hence, the contribution of exogenous variable $\rvx_2$ cannot be properly captured. Apparently, the assumption that intermediate vertices have no endogenous contributions is misaligned with real-world feature attribution scenarios. 


In this paper, we explore the potential to enhance the reasonableness of attribution methods through an advanced investigation of feature interactions. Specifically, we incorporate fine-grained causal relationships under the assumption that features form a DAG, a widely adopted and arguably fundamental assumption in causal modeling. Our contributions can be summarized as follows. 
\begin{enumerate}[align=left, leftmargin=0pt, labelindent=\parindent, listparindent=\parindent, labelwidth=0pt, itemindent=!, itemsep=0pt, topsep=0pt, parsep=0pt]
\item We propose DAG-SHAP, an edge-intervention-based attribution method that performs fine-grained, pathway-specific interventions on selected parent--child transmissions in a causal DAG, without perturbing other edges in the graph. 
\item We show that DAG-SHAP is the only method that simultaneously adheres to the four critical axiomatic properties required for feature attribution: causality, efficiency, externality, and exogeneity.
\item We introduce two efficient methods for DAG-SHAP computation: (i) an exact formulation via feature-distribution inference under edge interventions, and (ii) a Monte Carlo approximation over valid topological edge orderings, enabling the method to remain efficient and scalable in practice. 
\item Comprehensive experiments on both real and synthetic datasets demonstrate that DAG-SHAP achieves superior attribution performance compared to baseline methods. 
\end{enumerate}

\partitle{Organization} Section~\ref{sec:related_works} reviews the related work. Section~\ref{sec:preliminaries} describes the preliminaries. Section~\ref{sec:MotivationExample} illustrates our motivation with a toy example. Section~\ref{sec:method} presents the definition of DAG-SHAP. Section~\ref{sec:computation_DAG-SHAP} provides the algorithm to approximate DAG-SHAP. Section~\ref{sec:experiment} presents
the experiments. Finally, Section~\ref{sec:conclusion} draws a conclusion and discusses future directions.

\section{Related Work}\label{sec:related_works}
In this section, we discuss related work on classic feature attribution and causal feature attribution methods. 
\subsection{Classic Feature Attribution Techniques}
LIME (Local Interpretable Model-agnostic Explanations)\cite{ribeiro2016should} is a popular model explanation method which generates a local and interpretable model for each prediction made by a complex model. The advantage of LIME lies in its universality and simplicity, as it can be applied to any model and provides intuitive explanations of feature importance. Grad-CAM~\cite{selvaraju2017grad} highlights the areas of the image that contribute the most to the prediction of a specific class through visualization. Specifically, it uses the feature maps of the last convolutional layer and the gradient information regarding the prediction of a specific class to generate a Class Activation Map. DeepLIFT~\cite{shrikumar2017learning} explains the decision-making process of deep learning models through the differences in activation functions. SmoothGrad~\cite{smilkov2017smoothgrad} is a method designed to improve the interpretability of deep learning models, particularly for image classification tasks. It aims to enhance the quality of model prediction explanations by reducing the visual noise in gradient sensitivity maps. Integrated Gradients (IG)~\cite{sundararajan2017axiomatic} involves integrating the gradients of a model's output with respect to its inputs along the path from a baseline to the actual input data. Like Shapley-value-based attribution, Integrated Gradients has been proven to satisfy axiomatic properties including linearity, implementation invariance, sensitivity, and dummy, and it is independent of the model implementation. Layer-wise Relevance Propagation (LRP)~\cite{montavon2019layer} is a popular method for explaining neural networks. It explains the predictions of neural networks by propagating relevance scores from the output layer back to the input layer. MCI measures feature contribution based on the maximum marginal contribution a feature can bring, which cannot capture the right causal contribution. It gives some new properties like Super-efficiency and Sub-additivity~\cite{catav2020marginal}.
Beyond these model-explanation techniques, the database community has also used Shapley values to quantify contribution scores of data units in learning-oriented tasks~\cite{DBLP:journals/pvldb/JiaDWHGLZSS19,DBLP:journals/pvldb/XiaL0Q00023}. This line of work is conceptually related to our goal of fair contribution allocation, but differs in both attribution granularity and semantics. Existing database-oriented Shapley valuation methods typically regard records, data points, or datasets as players and measure their marginal effect on a downstream utility. In contrast, DAG-SHAP treats directed causal transmissions between observed features as attribution objects. Rather than asking how much a tuple or dataset contributes to a model, DAG-SHAP asks how much an intervened feature edge contributes to an instance-wise prediction under a causal DAG.

\subsection{Causal Feature Attribution Techniques}
As a measure of causal contribution, do-Shapley value~\cite{pmlr-v162-jung22a} provides a theoretical justification through axiomatic foundations. Like causal Shapley value~\cite{frye2020asymmetric}, it employs interventions but generalizes previous approaches to measure the causal contributions of each feature to a target effect induced by a black-box/unknown/inaccessible model. PWSHAP (Path-Wise Shapley effects)~\cite{ter2023pwshap} is a method for explaining the impact of specific binary variables, such as treatment effects or ethnicity in policy models, within predictive models. It evaluates the targeted effects in complex outcome models by combining a predictive model with a user-defined directed acyclic graph. Inspired by causal inference and randomized experiments, researchers have developed an algorithm to estimate AME (Average Marginal Effect)~\cite{lin2022measuring}, a measure of the expected average marginal effect of adding a data point to a subset of the training data sampled from a uniform distribution. CF-SHAP~\cite{albini2022counterfactual} is a method that combines counterfactual information for feature attribution. It strengthens and clarifies the link between actionable recourse and feature attributions, playing a role in advancing the development of causal feature attribution. Janzing et al.~\cite{janzing2020feature} primarily focus on distinguishing between calculating Shapley values based on observational versus interventional conditional distributions. They were among the first to emphasize the role of causality in feature attribution. Besides, Janzing et al.~\cite{janzing2024quantifying} propose Shapley-ICC, which uses structure-preserving interventions and Shapley-based symmetrization to obtain a \emph{global} decomposition of the target variable's uncertainty, such as variance or Shannon entropy. This objective and its modeling requirements differ from ours. DAG-SHAP operates on a causal DAG over observed features and estimates post-edge-intervention distributions over the observed variables. It therefore does not require identifying SCM-level exogenous noise variables. In contrast, Shapley-ICC attributes uncertainty through the exogenous noises $N_i$ in an SCM. Since such $N_i$ are generally not identifiable from a DAG and observational $P(X)$ alone, Shapley-ICC relies on stronger structural assumptions. Consequently, without additional SCM assumptions, Shapley-ICC is not applicable to our instance-wise DAG-based feature attribution setting.

Overall, classic feature attribution methods provide useful explanations but are typically not grounded in causal semantics, while existing causal Shapley-style methods mainly rely on node interventions and cannot capture fine-grained interactions along individual causal links in a feature DAG. These limitations motivate our proposed DAG-SHAP, which performs \emph{edge-based interventions} to obtain instance-wise attributions that respect the DAG structure while explicitly accounting for externality and exogenous contributions.

\section{Preliminaries}\label{sec:preliminaries}

In this section, we review the formal definitions of the feature attribution problem, the Shapley value, and fundamental attribution properties. 

\subsection{Feature Attribution}

Let $f: \mathbb{R}^n\rightarrow \mathbb{R}$ denote a trained ML model defined over the feature set $\sV=\{X_1,\dots,X_n\}$. For an arbitrary input vector $\vx=(\vx_1,\dots,\vx_n) \in \mathbb{R}^n$, the objective of feature attribution is to assign to each component $\vx_i$ a real-valued score $\phi_i\in \mathbb{R}$, referred to as its \textit{attribution value}, reflecting the contribution of $\vx_i$ to the model's output $f(\vx)$. With these attribution values, $f(\vx)$ can be decomposed as
\begin{equation}
    f(\vx) = f_0 + \sum\nolimits_{i=1}^n \phi_i,
\end{equation}
where $f_{0}\in \mathbb{R}$ is known as the \textit{baseline}, which may vary depending on the attribution goal. 

\subsection{Shapley Value}
Consider a set of players $\sP=\{1,\ldots,n\}$. A \emph{coalition} $\mathcal{S}$ is a subset of $\sP$ that cooperates to complete a task. A utility function $\mathcal{U}(\mathcal{S})$ $(\mathcal{S} \subseteq \sP)$ is the utility of coalition $\mathcal{S}$ for a task. The \emph{marginal contribution} of player $i$ with respect to a coalition $\mathcal{S}$ is $\mathcal{U}(\mathcal{S}\cup \{i\})-\mathcal{U}(\mathcal{S})$. Shapley value (SV) is the unique metric that satisfies the properties of fair reward allocation, including balance, symmetry, additivity, and zero element~\cite{winter2002shapley}. It measures the expectation of the marginal contribution of $i$ across all possible coalitions, that is,
\begin{equation}\label{equ:SV}
  \mathcal{SV}_i=\frac{1}{n} \sum_{\mathcal{S}\subseteq \sP \setminus \{i\}}\frac{\mathcal{U}(\mathcal{S}\cup \{i\})-\mathcal{U}(\mathcal{S})}{\binom{n-1}{|\mathcal{S}|}}.
\end{equation} 
According to Eq.~(\ref {equ:SV}), we can find that computing the exact Shapley value requires enumerating all utilities for all player subsets. Therefore, the computational complexity of exactly calculating the Shapley value is exponential~\cite{deng1994complexity}.
In the context of feature attribution, the feature set $\sV$ can be interpreted as the player set $\sP$, thereby enabling the SV-based methods to determine attribution values.

\subsection{Fundamental Properties}\label{sec:fundamental_properties}

\textit{Linearity}, \textit{Implementation Invariance}, \textit{Sensitivity}, and \textit{Dummy} are standard properties for feature attribution methods. These properties are generally straightforward for the attribution methods considered in this paper to satisfy, and thus they are not the main source of distinction among the methods. For a feature $X_i$, we denote its attribution value by $\phi_i$, and these properties can be stated as follows.

\begin{itemize}[align=left, leftmargin=0pt, labelindent=\parindent, listparindent=\parindent, labelwidth=0pt, itemindent=!, itemsep=0pt, topsep=0pt, parsep=0pt]
\item \textit{Linearity}: For any two value functions $f_u$ and $f_w$, and their sum $f=f_u+f_w$, the attributed value of each feature in $f$ should be equal to the sum of its attributed values in $f_u$ and $f_w$, i.e., $\phi_i(f)=\phi_i(f_u)+\phi_i(f_w)$.
\item \textit{Implementation Invariance}: The attribution results from an attribution method should be identical for two models if their outputs are the same for all inputs, even though their implementations may differ significantly.
\item \textit{Sensitivity}: If two inputs differ in only one feature and the model predictions for these two inputs are different, then the attribution for that differing feature should be non-zero. Formally, let $\bm{x}$ and $\bm{x}'$ be two inputs that differ only in feature $X_i$, and the corresponding model predictions satisfy $f(\bm{x})\neq f(\bm{x}')$. Then, the attribution $\phi_i$ should satisfy $\phi_i\neq 0$.
\item \textit{Dummy}: The attributed value of a feature should be zero if it has no exogenous influence on the outcome.
\end{itemize}

DAG-SHAP satisfies the four standard properties above, which serve as fundamental requirements for feature attribution methods. However, these properties are not sufficiently discriminative in causal feature attribution over DAGs, since they do not explicitly characterize whether an attribution method respects causal transmissions, captures cooperative effects involving descendants, or isolates exogenous influences. Therefore, we further focus on the following four properties, which are more challenging and central to our setting.

\begin{itemize}[align=left, leftmargin=0pt, labelindent=\parindent, listparindent=\parindent, labelwidth=0pt, itemindent=!, itemsep=0pt, topsep=0pt, parsep=0pt]
\item \textit{Causality}: The attribution value of each feature should be determined by its causal effect on the model output, i.e., by changes in $f$ under appropriate interventions, rather than by observational correlations induced by dependencies among features. 
\item \textit{Efficiency}: The sum of attribution values across all features should be equal to the difference between the model output and the baseline. 
\item \textit{Externality}: The attribution value of each feature should be evaluated across coalitions that include all its descendants, ensuring that cooperative effects are fully captured. 
\item \textit{Exogeneity}: The attribution value of each feature should be attributed solely to its exogenous influence, without being confounded by correlations or dependencies with other features. 
\end{itemize}

Note that the \textit{exogenous influence} of a feature refers to the portion of the overall effect on the model's output that can be directly attributed to the feature itself. 


\section{Problem Formulation \& Motivation}\label{sec:MotivationExample}

We begin by presenting the problem formulation addressed in this paper and subsequently discussing the limitations of existing SV-based feature attribution methods.

\subsection{Problem Formulation}

In this paper, we focus on attribution relative to the widely used baseline $f_0 = \mathbb{E}[f(\rvx)]$, where $\rvx$ is a random feature vector drawn from the input data distribution. This baseline represents the model's expected prediction over the population distribution, providing a natural reference point that quantifies how much an individual input deviates from the average prediction.
Specifically, we model the causal effects of the features as a directed acyclic graph (DAG) $\mathcal{G}=(\sV,\bm{E})$, where each vertex $X_i\in \sV$ represents a feature and each directed edge $\re_i=(p_i,c_i,X_{p_i})$, $p_i,c_i\!\in\! \{1,\!\cdots\!,\!n\}$, denotes the direct causal influence of feature $X_{p_i}$ (the parent) on feature $X_{c_i}$ (the child). The attribution value assigned to each feature must isolate its exogenous influence on the model output, ensuring that both its direct effect and its indirect causal effect are fully captured. Although our formulation uses causal graphs and interventions, our goal is local feature attribution rather than population-level causal effect estimation. We return to this distinction at the end of this section.

\subsection{Limitations of Previous Studies}

\partitle{Running Example}
Consider the toy example in Figure~\ref{fig:toy_example}. The feature set is given by $\sV=\{X_1,X_2\}$, and $\rvx_1, \rvx_2$ denote the exogenous influences of $X_1$ and $X_2$, respectively.
Feature $X_1$ has both a \textit{direct} influence on the model's output $Y=f(X_1,X_2)$ and an \textit{indirect} influence through its causal effect on $X_2$. 
Similarly, $X_2$ influences $Y$ via its own exogenous component $\rvx_2$ and by transmitting the indirect influence of $X_1$.
Given a specific explained input $\vx^*=(\vx_1^*,\vx_2^*)=(0.2,0.8)$, we aim to assign to each $\vx_i^*$ a value $\phi_i$ with respect to their uniform distributions.
\color{black}

In what follows, we discuss the limitations of existing SV-based feature attribution methods using this running example. Specifically, we focus on seven representative approaches that are based on distinct Shapley value variants, as shown in Table \ref{tab:comparison}. 
Due to page constraints, detailed descriptions of these methods are provided in Appendix Section~\ref{sec:mostrealted} of the full paper~\cite{dagshapfullpaper}.

\partitle{Off-Manifold Shapley Value}
As one of the pioneering SV-based feature attribution methods, off-manifold SV assumes feature independence, with the utility function $\mathcal{U}$ constructed to break the structural dependence of features. 
For example, when feature $X_1$ is explicitly specified as $\vx_1^*=0.2$ while the others are left unspecified, the marginal contribution of $X_1$ to the output $Y$ is 
\begin{equation*}
\mathbb{E}[f(X_1,X_2)\mid X_1=\vx_1^*,X_2=\rvx_1\!+\!\rvx_2]-\mathbb{E}[f(\rvx_1,\rvx_1\!+\!\rvx_2)].
\end{equation*}
This marginal contribution does not incorporate the indirect effect of $X_1=0.2$ since the value of $X_2$, i.e., $\rvx_1+\rvx_2$, is randomly sampled from its underlying distribution, which conflicts with the fact that feature $X_2$ is causally determined by feature $X_1$.

\partitle{On-Manifold Shapley Value}
While the on-manifold Shapley value incorporates feature dependency into its formulation, 
it suffers from causal reversion. 
For example, when $X_2$ is explicitly specified as $\vx_2^*=0.8$ while the others remain unspecified, the marginal contribution of $X_2$ is
\begin{equation*}
\mathbb{E}\!\left[f(X_1,X_2)\mid X_1=\rvx_1,X_2=\vx_2^*\right]-\mathbb{E}\!\left[f(\rvx_1,\rvx_1+\rvx_2)\right].
\end{equation*}
Conditioning on $X_2=0.8$ induces a posterior distribution over $X_1$, even though $X_2$ cannot influence $X_1$ from a causal perspective. Hence, this conditional expectation attributes to $\vx_2^*$ a part of the contribution that does not causally belong to $\vx_2^*$.

\partitle{Symmetry Causal Shapley Value}
It faces the same causal reversion issue, since the utility for $\vx_2^*=0.8$ is evaluated via an intervention that is not aligned with the data generation process.
Considering the marginal contribution of $\vx_2^*=0.8$ when cooperating with the empty set, it yields
\begin{equation*}
\mathbb{E}[f(\rvx_1,X_2)\mid \operatorname{do}(X_2=0.8)]-\mathbb{E}[f(\rvx_1,\rvx_1+\rvx_2)].
\end{equation*}
From a causal perspective, the event $X_2=0.8$ may rely on $X_1=0.2$ in the explained input. Without $\vx_1^*=0.2$, forcing $X_2$ to be $0.8$ is not a coherent intervention for evaluating the utility of $\vx_2^*$, and thus leads to an unreasonable attribution.

\partitle{Asymmetric (Causal) Shapley Value}
Both methods compute marginal contributions only over permutations in which the ancestor vertices appear before their descendants. This ordering constraint prevents an ancestor from receiving cooperative gains from its descendants, thereby violating the externality.
In our toy example, $\vx_1^*=0.2$ has a direct influence on $Y$, but it cannot receive a marginal contribution from cooperating with $\vx_2^*$ because $\vx_1^*$ must appear before $\vx_2^*$ in all permitted permutations. As a result, the attribution of $\vx_1^*$ becomes independent of $\vx_2^*$, contradicting the externality principle in cooperative game theory~\cite{shapley1953value}.

\partitle{Shapley Flow}
While satisfying \emph{cut-efficiency} (i.e., the sum of attributions within each cut equals the prediction change induced by that cut), Shapley flow does not satisfy the global efficiency axiom in Shapley's original formulation, where the sum of attributions over all features should equal the total change~\cite{shapley1953value}.
Moreover, Shapley flow enforces a cut-efficiency style conservation, which makes attributions ``flow'' through the graph and balance across cuts. Consequently, this approach implicitly assumes that only vertices with no incoming edges can introduce exogenous contributions, while intermediate vertices merely transmit contributions inherited from their ancestors. In our toy example, $X_2$ has its own exogenous influence $\rvx_2$ ($X_2=X_1+\rvx_2$), but these methods cannot properly isolate and attribute this independent exogenous part.

\partitle{Recursive Shapley Value}
The top-down hierarchical decomposition of the recursive Shapley value processes parent nodes prior to their descendants, eliminating the dependence of a parent node's attribution on its child nodes. Therefore, it does not satisfy the externality property. Similar to Shapley Flow, it also satisfies cut-efficiency. Consequently, it shares the same limitation that it cannot properly attribute the exogenous influences of non-root variables (i.e., it also violates exogeneity). Moreover, by enforcing cut-efficiency, it fails to satisfy the efficiency property.

\begin{table}[t]
\caption{Comparison of feature attribution method in terms of having (\cmark) and missing 
(\xmark) desiderata. }\label{tab:propertyTable}
\vspace{-1em}
\centering
\begingroup
\footnotesize
\setlength{\tabcolsep}{2.4pt}
\renewcommand{\arraystretch}{1.08}
\begin{tabular}{@{}lcccc@{}}
\toprule
\textbf{Method} &  \textbf{Causality} & \textbf{Efficiency} & \textbf{Externality} & \textbf{Exogeneity} \\ \midrule
Off-manifold SV~\cite{scott2017unified} &  \xmark  & \cmark & \cmark & \cmark \\ 
On-manifold SV~\cite{scott2017unified} &  \xmark & \cmark & \cmark & \cmark \\ 
Asymmetric SV~\cite{frye2020asymmetric} & \cmark  & \cmark &  \xmark& \cmark \\ 
Sym. causal SV~\cite{DBLP:conf/nips/HeskesSBC20} & \xmark & \cmark & \cmark & \cmark\\
Asym. causal SV~\cite{frye2020asymmetric} & \cmark &  \cmark & \xmark & \cmark \\ 
Shapley Flow~\cite{DBLP:conf/aistats/WangWL21} &  \cmark & \xmark &  \cmark & \xmark \\ 
Recursive SV~\cite{DBLP:conf/icml/SingalMN21} & \cmark & \xmark & \xmark & \xmark\\ 
DAG-SHAP (Ours) &  \cmark & \cmark  & \cmark & \cmark \\ \bottomrule
\end{tabular}
\endgroup
\vspace{-1em}
\label{tab:comparison}
\end{table}

\subsection{Feature Attribution vs. Causal Effect Estimation}

Since our formulation uses causal graphs and interventions, it is necessary to clarify that the target of this paper is still feature attribution, rather than causal effect estimation itself. The attribution problem considered in this paper is inherently local: it explains why an individual input deviates from the baseline prediction, rather than estimating a population-level average causal effect. In contrast, measures such as Average Causal Effect (ACE), direct effect, and indirect effect quantify how changing a feature affects the outcome on average across a population or along a specified causal path. They are useful causal quantities, but they do not allocate the collaborative contribution of a specific feature tuple to a particular model prediction. This distinction matters because local feature attribution must explain the specific interaction of feature values in one data point~\cite{schamberg2020direct}.

For example, suppose a model predicts the probability of a heart attack, denoted by $R$, from two features: body-mass index $M$ and average sleep duration $L$. Consider a baseline individual with $M=25$ and $L=7$ hours, and an explained individual with $M=30$ and $L=5$ hours. If ACE estimates that a one-unit increase in $M$ raises the probability by $1\%$ and reducing $L$ by one hour raises it by $2\%$, simply adding these global average effects would estimate a $9\%$ increase ($5$ BMI units $\times 1\%$ plus $2$ hours less sleep $\times 2\%$). However, because body-mass index and sleep duration may have synergistic effects, the prediction increase for this individual may be $11\%$ rather than the additive $9\%$ suggested by ACE. ACE describes average responses to isolated changes, but it does not determine how this instance-level $11\%$ prediction change should be distributed between $M$ and $L$. A local attribution method should allocate this joint effect between the two features, for example assigning $6.5\%$ to $M$ and $4.5\%$ to $L$, thereby explaining the specific prediction instead of replacing attribution with a global causal effect estimate. This is the setting in which causal knowledge and Shapley-style local contribution allocation need to be combined.

\partitle{Conclusion} As discussed above, existing SV-based attribution methods fail to simultaneously satisfy the four key properties of causality, efficiency, externality, and exogeneity. 
In contrast, we propose a novel approach that satisfies all four properties, the details of which are presented in the next section.

\section{DAG-SHAP: Edge Intervention Causal Shapley value}\label{sec:method}
We propose \textbf{DAG-SHAP}, an edge-based feature attribution method suitable for the scenario where the causal relations of input features can be formulated as a DAG. DAG-SHAP enables us to gain a deeper understanding of how features collaboratively contribute to the model output using the edge intervention to capture the causal influence within the features. 

\partitle{Edge Intervention}
Consider a DAG $\mathcal{G}=(\sV,\bm{E})$ with a set of vertices $\sV=\{X_1,\!\cdots\!,X_n\}$ and directed edges $\bm{E}=\{\re_1,\!\cdots\!,\re_m\}$. 
Given a specific explained input $\vx$, we instantiate the edges as $E(\vx)=\{e_1,\!\cdots\!,e_m\}$, where each instantiated edge $e_i=(p_i,c_i,\vx_{p_i})$ fixes the value transmitted from $p_i$ to $c_i$\footnote{For notational convenience, we abbreviate $X_{p_i},X_{c_i}$ as $p_i,c_i$ throughout the remainder of this paper.} to be $\vx_{p_i}$.
For a set of intervened edge indices $\mathcal{S}\subseteq\{1,\!\cdots\!,m\}$, we denote an \textbf{edge intervention} by
$$\operatorname{do}(\rve_{\mathcal{S}}=e_{\mathcal{S}}),$$
which enforces that, for every $i\in\mathcal{S}$, the causal influence from $p_i$ to $c_i$ is computed using the fixed value $\vx_{p_i}$, 
while leaving all other incoming edges of $c_i$ and all other outgoing edges of $p_i$ unchanged.
In this way, edge interventions enable fine-grained, pathway-specific causal analysis by intervening on selected parent-child transmissions without globally altering the entire set of child vertices of the same parent.

\partitle{DAG-SHAP}
Let $\Pi$ denote the set of all topological orderings of a given instantiated edge set $E(x)$. Each ordering $\pi\in \Pi$ respects the causal precedence constraints among edges, thereby specifying a valid sequence in which edge interventions can be applied. 
We introduce the formal definition of \textbf{edge-intervention causal Shapley value} as follows.

\begin{definition}[Edge-Intervention Causal Shapley Value]
\label{def:dag_shap}
Given the set $\Pi$. Let $\pi(i)$ denote the position of edge $e_i$ in the ordering $\pi\in \Pi$, and the sets 
$$\mathcal{S}_{\pi}^i=\{j\!\mid\!\pi(j)\!<\!\pi(i)\}, \quad \underline{\mathcal{S}}_{\pi}^i=\{j\!\mid\!\pi(j)\!\le\!\pi(i)\}$$
corresponding to the edges that precede $e_i$ and the edges up to and including $e_i$, respectively. The edge-intervention causal Shapley value of edge $e_i$ is given by
\begin{equation}\label{eq.edge_intervention}
\begin{aligned}
\Psi(e_i)=\frac{1}{|\Pi|}\sum_{\pi\in\Pi}\Big\{
&\mathbb{E}\!\left[f(\rvx)\mid \operatorname{do}(\rve_{\underline{\mathcal{S}}_{\pi}^i}=e_{\underline{\mathcal{S}}_{\pi}^i})\right] \\
&-\mathbb{E}\!\left[f(\rvx)\mid \operatorname{do}(\rve_{\mathcal{S}_{\pi}^i}=e_{\mathcal{S}_{\pi}^i})\right]
\Big\}.
\end{aligned}
\end{equation}

\end{definition}
This formulation implies that the attribution value of edge $e_i$ is the average marginal contribution of intervening on $e_i$, conditioned on having already intervened on all edges preceding $e_i$ in the ordering $\pi$. For brevity, we also refer to Definition~\ref{def:dag_shap} as \textbf{DAG-SHAP}. Unless otherwise specified, DAG-SHAP will be used throughout the remainder of this paper.

Let $\mathcal{O}_k$ be the set of edges with vertex $k$ as the parent, the DAG-SHAP of vertex $k$ is then defined as the sum of attribution values of its outgoing edges, which can be formulated as
\begin{equation}
   \Phi(k) = \sum_{e \in \mathcal{O}_k} \Psi(e).
\end{equation}

Detailed proofs showing that DAG-SHAP satisfies  all eight properties discussed in Section~\ref{sec:fundamental_properties} are provided in Appendix Section~\ref{sec:appendix_properties}. Since the aim of our paper is to derive the contribution of each feature, all the properties we prove pertain to the vertices of the DAG.

\partitle{Example}
We illustrate the definition of DAG-SHAP using the toy example in Figure~\ref{fig:toy_example}.
Denote the three edges as $\re_1: X_1\rightarrow X_2$, $\re_2: X_1\rightarrow Y$, and $\re_3: X_2\rightarrow Y$.
For the explained input $\vx^*=(0.2,0.8)$, the instantiated edges are
$e_1=(X_1,X_2,0.2)$, $e_2=(X_1,Y,0.2)$, and $e_3=(X_2,Y,0.8)$.
Since $e_1$ must precede $e_3$, the valid topological orderings are
\[
\Pi=\{(e_1,e_2,e_3),\ (e_1,e_3,e_2),\ (e_2,e_1,e_3)\}.
\]
For notational convenience, we define the utility function
\[
U(\mathcal{S}) := \mathbb{E}\!\left[f(\rvx)\mid \operatorname{do}(\rve_{\mathcal{S}}=e_{\mathcal{S}})\right],
\qquad U(\emptyset)=\mathbb{E}[f(\rvx)].
\]
Consider the ordering $\pi'=(e_1,e_2,e_3)$. The required utilities are
\begin{align*}
U(\{1\}) &= \mathbb{E}_{\overline{\rvx}_2\sim P(\rvx_2\mid \vx_1^*=0.2)}\!\left[f(\rvx_1,\overline{\rvx}_2)\right], \\
U(\{1,2\}) &= \mathbb{E}_{\overline{\rvx}_2\sim P(\rvx_2\mid \vx_1^*=0.2)}\!\left[f(0.2,\overline{\rvx}_2)\right], \\
U(\{1,2,3\}) &= \mathbb{E}\!\left[f(0.2,0.8)\right].
\end{align*}
Thus, the marginal contributions along $\pi'$ are
\begin{align*}
\Delta_{\pi}(e_1) & =U(\{1\})-U(\emptyset),\\
\Delta_{\pi}(e_2) & =U(\{1,2\})-U(\{1\}),\\
\Delta_{\pi}(e_3) & =U(\{1,2,3\})-U(\{1,2\}).
\end{align*}

We have computed the marginal contributions for the ordering $\pi' = (e_1, e_2, e_3)$. 
The marginal contributions for the other two valid orderings, $\pi'' = (e_1, e_3, e_2)$ and $\pi''' = (e_2, e_1, e_3)$, are computed analogously. 
Finally, the DAG-SHAP $\Psi(e_i)$ is obtained by averaging $\Delta_{\pi}(e_i)$ over all three orderings $\pi \in \Pi = \{\pi', \pi'', \pi'''\}$, and the vertex attribution is $\Phi(k)=\sum_{e\in\mathcal{O}_k}\Psi(e)$.

\begin{remark} 
Causal Shapley value and do-Shapley value~\cite{jung2022measuring, parafita2025practical} both use node interventions, with the key difference being whether the goal is to explain the impact on the data generation process of $Y$ or on a predictive model $f$. This is reflected in their utility definitions: for do-Shapley, $v(S) = \mathbb{E}\left[\mathbf{Y} \mid \operatorname{do}(\mathbf{x}_S = x_S)\right]$, and for Causal Shapley, $v(S) = \mathbb{E}\left[f(\mathbf{x}) \mid \operatorname{do}(\mathbf{x}_S = x_S)\right]$. Node interventions can be further classified into symmetric-sampling node interventions and asymmetric-sampling node interventions.
\end{remark} 

\partitle{Edge Intervention vs. Node Intervention}
The running example in Section~\ref{sec:MotivationExample} illustrates the limitations of existing Shapley-style attribution methods. We now use a minimal example to further clarify the key methodological difference between edge intervention and node intervention. Node interventions treat an entire feature value as the intervention object, and therefore may mix a feature's own exogenous influence with the influence inherited from its ancestors. In contrast, edge interventions operate on individual causal transmissions, allowing DAG-SHAP to separate inherited effects from exogenous effects before aggregating them to feature-level attributions. To avoid additional complications from graph constraints, we focus here on symmetric node-intervention, which is the closest node-level counterpart to our Shapley-style edge intervention, leaving the asymmetric case to Appendix Section~\ref{sec:comparasion_asy_Intervention} in the full paper.

We use the following example to show the difference. Let $X_1=\mathbf{x}_1$, where $\mathbf{x}_1$ is a random variable uniformly distributed on $[0,1]$, representing the exogenous influence of $X_1$; let $X_2=X_1 + \mathbf{x}_2$, where $\mathbf{x}_2$ is another random variable uniformly distributed on $[0,1]$, representing the exogenous influence of $X_2$. The target is generated as $Y=\max(X_1, X_2)$. We aim to attribute values to the explained input $x^*=[x_1^*,x_2^*]=[1,2]$ with respect to the baseline $[0,0]$.

For symmetric node intervention, both sampling permutations $(x_1^*,x_2^*)$ and $(x_2^*,x_1^*)$ are valid. The marginal contributions of $x_1^*$ and $x_2^*$ in these two permutations are shown in Table~\ref{tab:sym_node_marginal}, where each row explicitly keeps the corresponding intervention contrast. Thus, the attribution value for $x_1^*$ under symmetric node intervention is $(3/2+0)/2=3/4$, while the attribution value for $x_2^*$ is $(2+1/2)/2=5/4$. This result is misleading: although $X_1$ directly influences $X_2$ and therefore also affects $Y$ through $X_2$, the node intervention on $X_2$ mixes the exogenous contribution of $X_2$ with the influence inherited from $X_1$.


\begin{table*}[t]
\centering
\caption{Symmetric-sampling node interventions: marginal contributions under $(x_1^*,x_2^*)$ and $(x_2^*,x_1^*)$.}
\label{tab:sym_node_marginal}
\begingroup
\small
\setlength{\tabcolsep}{4pt}
\renewcommand{\arraystretch}{1.25}
\begin{tabularx}{\textwidth}{@{}C{0.22\textwidth}>{\centering\arraybackslash}X C{0.16\textwidth}@{}}
\toprule
\textbf{Node and permutation} & \textbf{Marginal contribution} & \textbf{Value} \\
\midrule
$x_1^*$ in $(x_1^*,x_2^*)$ &
$\mathbb{E}\!\left[\mathbf{Y}\mid \operatorname{do}\!\left(\mathbf{x}_{\{1\}}=x_{\{1\}}^*\right)\right]
-\mathbb{E}\!\left[\mathbf{Y}\mid \operatorname{do}(\emptyset)\right]$ &
$\frac{3}{2}-0=\frac{3}{2}$ \\
$x_1^*$ in $(x_2^*,x_1^*)$ &
$\mathbb{E}\!\left[\mathbf{Y}\mid \operatorname{do}\!\left(\mathbf{x}_{\{1,2\}}=x_{\{1,2\}}^*\right)\right]
-\mathbb{E}\!\left[\mathbf{Y}\mid \operatorname{do}\!\left(\mathbf{x}_{\{2\}}=x_2^*\right)\right]$ &
$2-2=0$ \\
$x_2^*$ in $(x_2^*,x_1^*)$ &
$\mathbb{E}\!\left[\mathbf{Y}\mid \operatorname{do}\!\left(\mathbf{x}_{\{2\}}=x_{\{2\}}^*\right)\right]
-\mathbb{E}\!\left[\mathbf{Y}\mid \operatorname{do}(\emptyset)\right]$ &
$2-0=2$ \\
$x_2^*$ in $(x_1^*,x_2^*)$ &
$\mathbb{E}\!\left[\mathbf{Y}\mid \operatorname{do}\!\left(\mathbf{x}_{\{1,2\}}=x_{\{1,2\}}^*\right)\right]
-\mathbb{E}\!\left[\mathbf{Y}\mid \operatorname{do}\!\left(\mathbf{x}_{\{1\}}=x_{\{1\}}^*\right)\right]$ &
$2-\frac{3}{2}=\frac{1}{2}$ \\
\bottomrule
\end{tabularx}
\endgroup
\end{table*}

For DAG-SHAP edge interventions, the edges for the instance $x^*$ are denoted by $e_1^*:X_1\to X_2$, $e_2^*:X_1\to Y$, and $e_3^*:X_2\to Y$. The valid edge permutations are $(e_1^*,e_2^*,e_3^*)$, $(e_1^*,e_3^*,e_2^*)$, and $(e_2^*,e_1^*,e_3^*)$, since $e_1^*$ must precede $e_3^*$. The marginal contributions of each edge in each arrangement are shown in Table~\ref{tab:edge_intervention_marginals_sym}. The attribution value for $x_1^*$ in DAG-SHAP is $(3/2 + 3/2 + 1/4) / 3 + (0 + 0 + 5/4) / 3 = 3/2$, and the attribution value for $x_2^*$ is $(1/2 + 1/2 + 1/2) / 3 = 1/2$. Since $x_1^*=1$ directly influences $X_2$ and in turn influences $Y=\max(X_1,X_2)$ through $X_2$, $x_1^*$ should receive more attribution than $x_2^*$. DAG-SHAP obtains this result because it attributes the causal transmissions $X_1\to X_2$, $X_1\to Y$, and $X_2\to Y$ separately before aggregating outgoing-edge contributions to feature-level attributions. By contrast, symmetric node intervention gives the misleading result $\phi_2>\phi_1$ because it includes the inherited contribution from $X_1$ in the marginal contribution of $X_2$, i.e.,
$\mathbb{E}[Y \mid \operatorname{do}(X_2 = x_2^*)] - \mathbb{E}[Y \mid \operatorname{do}(\emptyset)]=2$.

\begin{table*}[t]
\centering
\caption{Edge interventions (DAG-SHAP): marginal contributions under all valid edge permutations.}
\label{tab:edge_intervention_marginals_sym}
\begingroup
\small
\setlength{\tabcolsep}{4pt}
\renewcommand{\arraystretch}{1.25}
\begin{tabularx}{\textwidth}{@{}C{0.22\textwidth}>{\centering\arraybackslash}X C{0.16\textwidth}@{}}
\toprule
\textbf{Edge and permutation} & \textbf{Marginal contribution} & \textbf{Value} \\
\midrule
$e_1^*$ in $(e_1^*,e_2^*,e_3^*)$ &
$\mathbb{E}\!\left[\mathbf{Y}\mid \operatorname{do}\!\left(\mathbf{e}_{\{1\}}=e_1^*\right)\right]
-\mathbb{E}\!\left[\mathbf{Y}\mid \operatorname{do}(\emptyset)\right]$ &
$\frac{3}{2}-0=\frac{3}{2}$ \\

$e_1^*$ in $(e_1^*,e_3^*,e_2^*)$ &
$\mathbb{E}\!\left[\mathbf{Y}\mid \operatorname{do}\!\left(\mathbf{e}_{\{1\}}=e_1^*\right)\right]
-\mathbb{E}\!\left[\mathbf{Y}\mid \operatorname{do}(\emptyset)\right]$ &
$\frac{3}{2}-0=\frac{3}{2}$ \\

$e_1^*$ in $(e_2^*,e_1^*,e_3^*)$ &
$\mathbb{E}\!\left[\mathbf{Y}\mid \operatorname{do}\!\left(\mathbf{e}_{\{1,2\}}=e_{\{1,2\}}^*\right)\right]
-\mathbb{E}\!\left[\mathbf{Y}\mid \operatorname{do}\!\left(\mathbf{e}_{\{2\}}=e_2^*\right)\right]$ &
$\frac{3}{2}-\frac{5}{4}=\frac{1}{4}$ \\

$e_2^*$ in $(e_1^*,e_2^*,e_3^*)$ &
$\mathbb{E}\!\left[\mathbf{Y}\mid \operatorname{do}\!\left(\mathbf{e}_{\{1,2\}}=e_{\{1,2\}}^*\right)\right]
-\mathbb{E}\!\left[\mathbf{Y}\mid \operatorname{do}\!\left(\mathbf{e}_{\{1\}}=e_1^*\right)\right]$ &
$\frac{3}{2}-\frac{3}{2}=0$ \\

$e_2^*$ in $(e_1^*,e_3^*,e_2^*)$ &
$\mathbb{E}\!\left[\mathbf{Y}\mid \operatorname{do}\!\left(\mathbf{e}_{\{1,2,3\}}=e_{\{1,2,3\}}^*\right)\right]
-\mathbb{E}\!\left[\mathbf{Y}\mid \operatorname{do}\!\left(\mathbf{e}_{\{1,3\}}=e_{\{1,3\}}^*\right)\right]$ &
$2-2=0$ \\

$e_2^*$ in $(e_2^*,e_1^*,e_3^*)$ &
$\mathbb{E}\!\left[\mathbf{Y}\mid \operatorname{do}\!\left(\mathbf{e}_{\{2\}}=e_2^*\right)\right]
-\mathbb{E}\!\left[\mathbf{Y}\mid \operatorname{do}(\emptyset)\right]$ &
$\frac{5}{4}-0=\frac{5}{4}$ \\

$e_3^*$ in $(e_1^*,e_2^*,e_3^*)$ &
$\mathbb{E}\!\left[\mathbf{Y}\mid \operatorname{do}\!\left(\mathbf{e}_{\{1,2,3\}}=e_{\{1,2,3\}}^*\right)\right]
-\mathbb{E}\!\left[\mathbf{Y}\mid \operatorname{do}\!\left(\mathbf{e}_{\{1,2\}}=e_{\{1,2\}}^*\right)\right]$ &
$2-\frac{3}{2}=\frac{1}{2}$ \\

$e_3^*$ in $(e_1^*,e_3^*,e_2^*)$ &
$\mathbb{E}\!\left[\mathbf{Y}\mid \operatorname{do}\!\left(\mathbf{e}_{\{1,3\}}=e_{\{1,3\}}^*\right)\right]
-\mathbb{E}\!\left[\mathbf{Y}\mid \operatorname{do}\!\left(\mathbf{e}_{\{1\}}=e_1^*\right)\right]$ &
$2-\frac{3}{2}=\frac{1}{2}$ \\

$e_3^*$ in $(e_2^*,e_1^*,e_3^*)$ &
$\mathbb{E}\!\left[\mathbf{Y}\mid \operatorname{do}\!\left(\mathbf{e}_{\{1,2,3\}}=e_{\{1,2,3\}}^*\right)\right]
-\mathbb{E}\!\left[\mathbf{Y}\mid \operatorname{do}\!\left(\mathbf{e}_{\{1,2\}}=e_{\{1,2\}}^*\right)\right]$ &
$2-\frac{3}{2}=\frac{1}{2}$ \\
\bottomrule
\end{tabularx}
\endgroup
\end{table*}

\section{Implementation of DAG-SHAP}\label{sec:computation_DAG-SHAP}

In this section, we introduce an exact computation method for edge intervention causal Shapley value by inferring feature distributions. Additionally, we propose an approximate algorithm for cases where these distributions are difficult to estimate in real datasets, as shown in Section~\ref{sec:approximate_DAG-SHAP}.

\subsection{Exact Computation of DAG-SHAP} \label{sec:proDistribution}

To compute the edge intervention Shapley value, the initial step involves calculating the value function \( \mathcal{U}(\mathcal{S}) = \mathbb{E}[f(\rvx) \mid \operatorname{do}(\rve_{\mathcal{S}} = e_{\mathcal{S}})] \), which necessitates determining the distribution of $\rvx$ under edge intervention. When intervening on the edges \( \rve_\mathcal{S} \), for those edges where the child vertex is the target label, the values of the parent vertices corresponding to these edges can be directly fixed by the feature values of the input $\vx$ since these edges represent the direct influence of features to the target label.  Let \( \mathbb{D}_{\mathcal{S}} \) denote the set of features that have a direct influence on the target label, and whose direct influence on the target label is through the edges in \( \mathbf{e}_\mathcal{S} \).  \( \sV_{\mathcal{S}} \) is the complement of \( \mathbb{D}_{\mathcal{S}} \). For the value function \( \mathcal{U}(\mathbf{\mathcal{S}}) = \mathbb{E}[f(\rvx) \mid \operatorname{do}(\mathbf{e}_{\mathcal{S}} = e_{\mathcal{S}})] \), the values of the features in \( \mathbb{D}_{\mathcal{S}} \) are determined by the values of the features in \( \bm{x} \), given that their direct influences are subjected to intervention. Consequently, the DAG-SHAP for edge \( e_i \) in Eq. (\ref{eq.edge_intervention}) can be reformulated as

\vspace{-1em}
\begin{equation*}
\begin{aligned}
\Psi_v(e_i)
&=\sum_{\pi \in \Pi} \frac{1}{|\Pi|}
\begin{aligned}[t]
&\Big\{\mathbb{E}[f(\rvx_{\sV_{\underline{\mathcal{S}}_{\pi}^i}}, \vx_{\mathbb{D}_{\underline{\mathcal{S}}_{\pi}^i}})
\mid\operatorname{do}(\rve_{\underline{\mathcal{S}}_{\pi}^i}\!=e_{\underline{\mathcal{S}}_{\pi}^i})] \\
&-\mathbb{E}[f(\rvx_{\sV_{\mathcal{S}_{\pi}^i}}, \vx_{\mathbb{D}_{\mathcal{S}_{\pi}^i}})
\mid\operatorname{do}(\rve_{\mathcal{S}_{\pi}^i}=e_{\mathcal{S}_{\pi}^i})]\Big\}.
\end{aligned}
\end{aligned}
\end{equation*}

To calculate $\mathbb{E} [f(\rvx_{\sV_{\mathcal{S}}}, \vx_{\mathbb{D}_{\mathcal{S}}})|\operatorname{do}(\rve_{\mathcal{S}}=e_{\mathcal{S}})]$, we should determine how to get the distribution of $\rvx_{\sV_{\mathcal{S}}}$. 
For a given DAG, the distribution of the data $\rvx$ generated by this graph satisfies the Markov property. Denote the distribution of the data as $P(\rvx)$, which can be expressed as follows
\begin{displaymath}
P(\rvx)=\Pi_{i \in \sV} P(\rvx_i|\rvx_{pa(i)}),    
\end{displaymath}
where $pa(i)$ represents the set of parent vertices of $i$ in the graph. When we intervene in the generation of data, the conditional distribution of each feature needs to be changed accordingly. The distribution formula of $\rvx_{\sV_\mathcal{S}}$ after the intervention is
\begin{displaymath}
P(\rvx_{\sV_\mathcal{S}}|\operatorname{do}(\rve_{\mathcal{S}}=e_{\mathcal{S}})) = \Pi_{i \in \sV_\mathcal{S}} P(\rvx_i|\rvx_{pa(i)\cap{\sV_{\mathcal{S}}^i}},\vx_{pa(i)\cap{\mathbb{D}_{\mathcal{S}}^i}}),    
\end{displaymath}
where $\mathbb{D}_\mathcal{S}^i$ is the set of features that have a direct influence on the value of feature $i$ and are connected to $i$ by edges in $\mathbf{e}_{\mathcal{S}}$. $\sV_\mathcal{S}^i$ is the complement of $\mathbb{D}_\mathcal{S}^i$. Since the permutations used to calculate the marginal contributions follow the topological order, the distributions of all parent vertices are already determined when calculating the distribution of a particular feature. We sequentially compute the distributions of all vertices. After calculating the marginal contribution for all topological orderings, a weighted average can be used to obtain the attribution value for each edge. By aggregating the attribution values of all outgoing edges of a vertex, we can compute the total attribution value for the corresponding feature point, reflecting its integrated contribution across all causal pathways.

\subsection{Approximation of DAG-SHAP}\label{sec:approximate_DAG-SHAP}

\begin{algorithm}[htbp]
\caption{Sampling $\rvx_{\sV_\mathcal{S}}$ under edge intervention.} \label{Alg:Edge_intervention}
\begin{algorithmic}[1]
\REQUIRE Edge permutation $\pi$, intervention edge set size $s$, explained input $\bm{x}$, baseline input $\mathcal{D}$.
\ENSURE \, A sample of $\rvx_{\sV_\mathcal{S}}$, denoted as $\overline{\rvx}_{\sV_\mathcal{S}}$.

\STATE Calculate the in-degree of each vertex in the graph and store it in the array $\vd$;
\STATE Randomly initialize $\overline{\rvx}_{\sV_\mathcal{S}}$ based on $\mathcal{D}$;
    
\FOR{$i = 1$ to length of $\pi$}
\STATE Let $c$ denote the child vertex of edge $\pi(i)$;
\STATE Let $p$ denote the parent vertex of edge $\pi(i)$;
\IF{$i \leq s$}
\STATE Assign value $\bm{x}_{p}$ to edge $p \rightarrow c$;
\ENDIF
\STATE  $\vd_c \gets \vd_c - 1$;
\IF{ $\vd_c$ equals $0$}
\STATE Infer the value of $\overline{\rvx}_{c}$ from the values of the edges \\ between it and its parent vertices;
\FOR{each child vertex $k$ of vertex $c$}
\STATE Assign value $\overline{\rvx}_{c}$ to edge $c \rightarrow k$; 
\ENDFOR
\ENDIF
\ENDFOR
    
\RETURN  $\overline{\rvx}_{\sV_\mathcal{S}}$.
\end{algorithmic}
\end{algorithm}

In Section~\ref{sec:proDistribution}, we describe the way to infer the distribution of each feature. However, we generally cannot accurately obtain the specific distribution expressions for each feature for real datasets. Therefore, we use approximation methods based on data sampling to estimate the value of each feature, rather than directly solving the exact form of the distribution. The detailed process is shown in Algorithm~\ref{Alg:Edge_intervention}. Given a valid edge ordering $\pi$ and an intervention budget $s$, we obtain the sampled value of $\rvx_{\sV_{\mathcal{S}}}$ by forward simulation on the DAG and denote it as  $\overline{\rvx}_{\sV_{\mathcal{S}}}$.
It starts from a baseline sample from $\mathcal{D}$, clamps the first $s$ edges in $\pi$ to the corresponding values in $\vx$, and sequentially infers each vertex once its parent-edge values are fixed, propagating the inferred value to its outgoing edges.
This yields one sample from the post-intervention distribution used to estimate $\mathcal{U}(\mathcal{S})$.

Since computing the Shapley value is an \#P-hard problem~\cite{deng1994complexity,zhang2023efficient}, the computational cost becomes very high when the number of edges increases. Existing efficient Shapley methods~\cite{DBLP:conf/icde/0006XSL0P023,DBLP:journals/pacmmod/LuoPXZX24} do not handle DAG-SHAP's topologically constrained edge permutations, where each marginal contribution depends on prior edge interventions. Therefore, we propose an approximation method based on Monte Carlo sampling for DAG-SHAP. The detailed process is shown in Algorithm~\ref{Alg:approximation-DAG-SHAP}. We first enumerate edge permutations that comply with the topological sort and then intervene on each edge to compute its marginal contribution (Lines 6-9). After obtaining the marginal contributions, we compute the weighted average to estimate the attribution value of the edge (Lines 13-14). Finally, we sum the attribution values of the outgoing edges of each feature to get the overall edge attribution estimation (Lines 16-19). The approximation error and scalability of this algorithm are evaluated in Sections~\ref{sec:efficiencyexp} and~\ref{sec:scalabilityexp}, respectively.

\begin{theorem}\label{theorem:DAG_SHAP_sampling}
    Denote by $\tau$ the number of sampled valid permutations. If each edge marginal contribution has range at most $r$, then the approximation based on Algorithm~\ref{Alg:approximation-DAG-SHAP} satisfies
    \[
        \begin{aligned}
    \mathbb{P}(|\overline{\Phi(i)}-\Phi(i)|\geq \epsilon)
    &\leq 2m\exp\!\left(-\frac{2\tau\epsilon^2}{m^2r^2}\right)\\
    &=\mathcal{O}\!\left(m\exp\!\left(-\frac{\tau}{m^2}\right)\right).
    \end{aligned}
    \]
\end{theorem}

\begin{proof}
For notational simplicity, index the outgoing edges of feature vertex $i$ by a subset of at most $m$ edges. Since the feature attribution $\Phi(i)$ is obtained by summing the attribution values of these outgoing edges, we have
\begin{align*}
    &\mathbb{P}(|\overline{\Phi(i)}-\Phi(i)|\geq \epsilon) \\
    \leq{}& \mathbb{P}\!\left(\sum_{j=1}^m|\overline{\Psi(j)}-\Psi(j)|\geq \epsilon\right) \\
    \leq{}& \sum_{j=1}^m \mathbb{P}\!\left(|\overline{\Psi(j)}-\Psi(j)|\geq \frac{\epsilon}{m}\right) \\
    \leq{}& 2\sum_{j=1}^m\exp\!\left(-\frac{2\tau(\frac{\epsilon}{m})^2}{(b_j-a_j)^2}\right) \\
    \leq{}& 2m\exp\!\left(-\frac{2\tau(\frac{\epsilon}{m})^2}{r^2}\right),
\end{align*}
where $(a_j,b_j)$ denotes the range of the marginal contribution of edge $j$, and $r=\max_{1\leq j\leq m}(b_j-a_j)$. The second inequality follows from the union bound, and the third inequality follows from Hoeffding's inequality. Therefore,
\begin{align*}
    \mathcal{O}\!\left(2m\exp\!\left(-\frac{2\tau(\frac{\epsilon}{m})^2}{r^2}\right)\right)
    &=\mathcal{O}\!\left(2m\exp\!\left(-\frac{2\tau\epsilon^2}{m^2r^2}\right)\right) \\
    &=\mathcal{O}\!\left(m \cdot \exp\!\left(-\frac{\tau}{m^2}\right)\right).
\end{align*}
Thus, we complete the proof.
\end{proof}

\begin{algorithm}[t]
\caption{Approximation of DAG-SHAP.}\label{Alg:approximation-DAG-SHAP}
\begin{algorithmic}[1]
\REQUIRE Graph $\mathcal{G}$, vertices $\sV$, edges $E$, explained input $\vx$, sampling number $T$.
\ENSURE Approximate attribution values of each edge and each feature.
\STATE Initialize counter $cnt \gets 0$;
\FOR{$t = 1$ to $T$}
\STATE  Let $\pi$ be a random permutation of $E$;
\IF{$\pi$ is a valid toposort}
\STATE $u \gets 0$, $cnt \gets cnt + 1$;
\FOR{$i = 1$ to length of $\pi$}
\STATE Obtain $\overline{\rvx}_{\sV_\mathcal{S}}$ by intervening on edges $\{\pi(1), \cdots, \pi(i)\}$ as described in Algorithm~\ref{Alg:Edge_intervention};
\STATE $\overline{\Psi(\pi(i))} \gets  \overline{\Psi(\pi(i))} + f(\overline{\rvx}_{\sV_\mathcal{S}},\overline{\rvx}_{\mathbb{D}_\mathcal{S}})-u$;
\STATE $u \gets f(\overline{\rvx}_{\sV_\mathcal{S}},\overline{\rvx}_{\mathbb{D}_\mathcal{S}})$;
\ENDFOR
\ENDIF
\ENDFOR
\FOR{$i = 1$ to $m$}
\STATE $\overline{\Psi(i)} \gets \overline{\Psi(i)} / cnt$;
\ENDFOR
\FOR{$i = 1$ to $n$}
\STATE Get $\overline{\Phi(i)}$ by summing the attribution values of \\ its outgoing edges;
\ENDFOR
\RETURN $\overline{\Psi(1)}$, $\cdots$, $\overline{\Psi(m)}$, $\overline{\Phi(1)}$, $\cdots$, $\overline{\Phi(n)}$.
\end{algorithmic}
\end{algorithm}

\section{Experiments}\label{sec:experiment}
To show the superiority of DAG-SHAP over existing methods, we conduct experiments on both synthetic and real datasets. We present the attribution values of the benchmark, DAG-SHAP, and the baseline algorithms using bar charts. Additionally, we calculate the mean absolute error between DAG-SHAP and the benchmark, as well as between each baseline algorithm and the benchmark. We utilize a Mixture Density Network to predict the distribution of child vertices after intervention on parent vertices within the input features, consistent with the approach used in causal Shapley value~\cite{DBLP:conf/nips/HeskesSBC20}. We experiment with two models for predicting the target feature of the input features: a DNN and an XGBoost model. The DNN and XGBoost attribution results are reported in the corresponding experimental subsections. Since the attribution scores for each method are obtained via approximation, we further analyze the approximation error in Section~\ref{sec:efficiencyexp} and the scalability of DAG-SHAP in Section~\ref{sec:scalabilityexp}.

\begin{figure}[t]
\centering
    \subfigure[Synthetic dataset.]{
        \begin{minipage}[b]{0.18\textwidth}
        \begin{tikzpicture}[node distance=1.1cm, on grid, auto]
            \node[draw, circle, thick] (X1) {$X_1$};
            \node[draw, circle, thick] (X2) [right=of X1] {$X_2$};
            \node[draw, circle, thick] (X3) [below left=of X1] {$X_3$};
            \node[draw, circle, thick] (X4) [below right=of X2] {$X_4$};
            \node[draw, circle, thick] (Y) [below =of X3, xshift=1.35cm] {$Y$};
            
            \path[->,thick] 
            (X1) edge (X3)
            (X1) edge (X4)
            (X1) edge (Y)
            (X2) edge (X3)
            (X2) edge (X4)
            (X2) edge (Y)
            (X3) edge (Y)
            (X4) edge (Y);
        \end{tikzpicture}
        \label{fig:DAG_synthetic}
    \end{minipage}%
    }
    \hspace{0.02\textwidth}
    \subfigure[Splitted synthetic dataset.]{
        \begin{minipage}[b]{0.23\textwidth}
        \begin{tikzpicture}[node distance=0.95cm, on grid, auto]
            \node[draw, circle, thick,inner sep=1.5pt] (X12) {$\tilde{X}_1^2$};
            \node[draw, circle, thick,inner sep=1.5pt] (X11) [below left=of X12,yshift=-0.25cm]{$\tilde{X}_1^1$};
            \node[draw, circle, thick,inner sep=1.5pt] (X13) [below right=of X12,yshift=0.25cm]{$\tilde{X}_1^3$};
            \node[draw, circle, thick,inner sep=1.5pt] (X23) [right=of X13] {$\tilde{X}_2^3$};
            \node[draw, circle, thick,inner sep=1.5pt] (X22) [above right=of X23,,yshift=-0.25cm] {$\tilde{X}_2^2$};
            \node[draw, circle, thick,inner sep=1.5pt] (X21) [below right=of X22,yshift=-0.25cm] {$\tilde{X}_2^1$};
            
            \node[draw, circle, thick,inner sep=1.5pt] (X3) [below=of X12,yshift=-0.35cm] {$\tilde{X}_3$};
            \node[draw, circle, thick,inner sep=1.5pt] (X4) [below=of X22,yshift=-0.35cm] {$\tilde{X}_4$};
            \node[draw, circle, thick] (Y) [below=of X13, xshift=0.50cm,yshift=-0.75cm] {$Y$};
            
            \path[->,thick] 
            (X12) edge (X3)
            (X22) edge (X4)
            (X13) edge (X4)
            (X23) edge (X3)
            (X3) edge (Y)
            (X4) edge (Y);
            \path[->, bend left, thick] 
            (X21) edge (Y);
            \path[->,bend right, thick]
            (X11) edge (Y);
        \end{tikzpicture}
        \label{fig:DAG_splitted_synthetic}
    \end{minipage}%
    }
    \caption{DAGs of the synthetic datasets.} 
\end{figure}
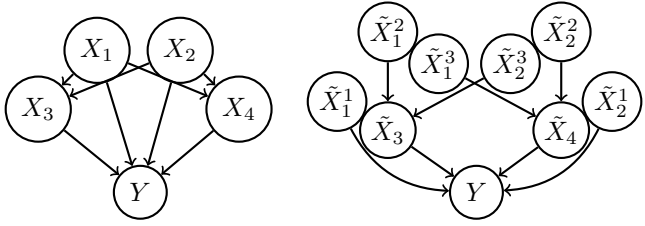

\subsection{Experiments on Synthetic Dataset}\label{sec:syntheticexp}
\partitle{Data Generation Process} Consider a synthetic dataset consisting of four input features, \(X_1\), \(X_2\), \(X_3\), and \(X_4\), along with a target feature \(Y\). The features \(X_1\) and \(X_2\) are independently sampled from two exogenous continuous random variables, \(\rvx_1\) and \(\rvx_2\), respectively. Both \(\rvx_1\) and \(\rvx_2\) follow a uniform distribution over the interval \([0, 10]\), i.e., \(\rvx_1, \rvx_2 \sim U(0, 10)\).
$X_3$ is influenced by $X_1$, $X_2$, and an additional exogenous variable $\rvx_3 \sim U(0, 10)$, and given by $X_3 = X_1 + X_2 +\rvx_3$. Similarly, $X_4$ is influenced by $X_1$, $X_2$, and another exogenous variable $\rvx_4 \sim U(0,10)$, and given by $X_4 = X_1 + X_2 +\rvx_4$. Finally, the target feature $Y$ is given by $Y = X_1 \cdot X_3 + X_2 \cdot X_4$.
The causal structure of this data generation process is shown in Figure~\ref{fig:DAG_synthetic}. 

\partitle{Desirable Attribution} To get desirable benchmark attribution, we conduct vertex splitting for the causal structure in Figure~\ref{fig:DAG_splitted_synthetic}. Specifically, we split the influence of \(X_1\) and \(X_2\) on \(X_3, X_4, \text{and}\ Y\).
Denote \(\tilde{X}_1^1\), \(\tilde{X}_1^2\), and \(\tilde{X}_1^3\) as copies of \(X_1\), while \(\tilde{X}_2^1\), \(\tilde{X}_2^2\), and \(\tilde{X}_2^3\) are copies of \(X_2\). The features \(\tilde{X}_3\) and \(\tilde{X}_4\) correspond to \(X_3\) and \(X_4\), respectively. Specifically, \(\tilde{X}_3\) can be expressed as \(\tilde{X}_3 = \tilde{X}_1^2 + \tilde{X}_2^2 + \tilde{\rvx}_3\) and \(\tilde{X}_4 = \tilde{X}_1^3 + \tilde{X}_2^3 + \tilde{\rvx}_4\), where \(\tilde{\rvx}_3 = \rvx_3\) and \(\tilde{\rvx}_4 = \rvx_4\). The target feature is defined as \(\tilde{Y} = \tilde{X}_1^1 \cdot \tilde{X}_3 + \tilde{X}_2^1 \cdot \tilde{X}_4\). The DAG is illustrated in Figure ~\ref{fig:DAG_splitted_synthetic}, following the data generation process. A key advantage of this split dataset is that the attributions of edges and vertices are equal as each vertex has only one outgoing edge, thereby eliminating externality issues in asymmetry causal Shapley value. Thus, asymmetric causal Shapley value can satisfy all desirable properties in Section~\ref{sec:method}. We train the used model to predict labels based on features \(X_1, X_2, X_3, \text{and}\ X_4\). This model can also predict labels when provided with values of \(\tilde{X}_1^1, \tilde{X}_2^1, \tilde{X}_3, \text{and}\ \tilde{X}_4\). Attribution is conducted using the model on \(\tilde{\bm{x}}\) with asymmetric causal Shapley value, where the values of \(\tilde{X}_1^2, \tilde{X}_2^2, \tilde{X}_1^3, \tilde{X}_2^3\) are not model inputs but serve to intervene on \(\tilde{X}_1^1, \tilde{X}_2^1, \tilde{X}_3, \tilde{X}_4\). The attribution results can benchmark the input tuple \(\bm{x}\) in the original dataset since \(\bm{x}_1 = \tilde{\bm{x}}_1^1 = \tilde{\bm{x}}_1^2 = \tilde{\bm{x}}_1^3\) and \(\bm{x}_2 = \tilde{\bm{x}}_2^1 = \tilde{\bm{x}}_2^2 = \tilde{\bm{x}}_2^3\). This is because the sum of attribution values for \(\tilde{\bm{x}}_1^1, \tilde{\bm{x}}_1^2, \tilde{\bm{x}}_1^3\) should be equal to the attribution value of \(\bm{x}_1\) since \(\tilde{\bm{x}}_1^1, \tilde{\bm{x}}_1^2, \tilde{\bm{x}}_1^3\) essentially represent the decomposition of \(\bm{x}_1\)'s influence on $Y$, and similarly for \(\tilde{\bm{x}}_2^1, \tilde{\bm{x}}_2^2, \tilde{\bm{x}}_2^3\) with \(\bm{x}_2\). Moreover, the attribution value of \(\tilde{\bm{x}}_3\) must be equal to that of \(\bm{x}_3\), and likewise for \(\tilde{\bm{x}}_4\) and \(\bm{x}_4\), in accordance with implementation invariance~\cite{sundararajan2017axiomatic} as they map the original input. 

\begin{figure}[htbp]
    \includegraphics[width=0.48\textwidth, height=0.2\textheight,trim=2cm 0 0 0, clip]{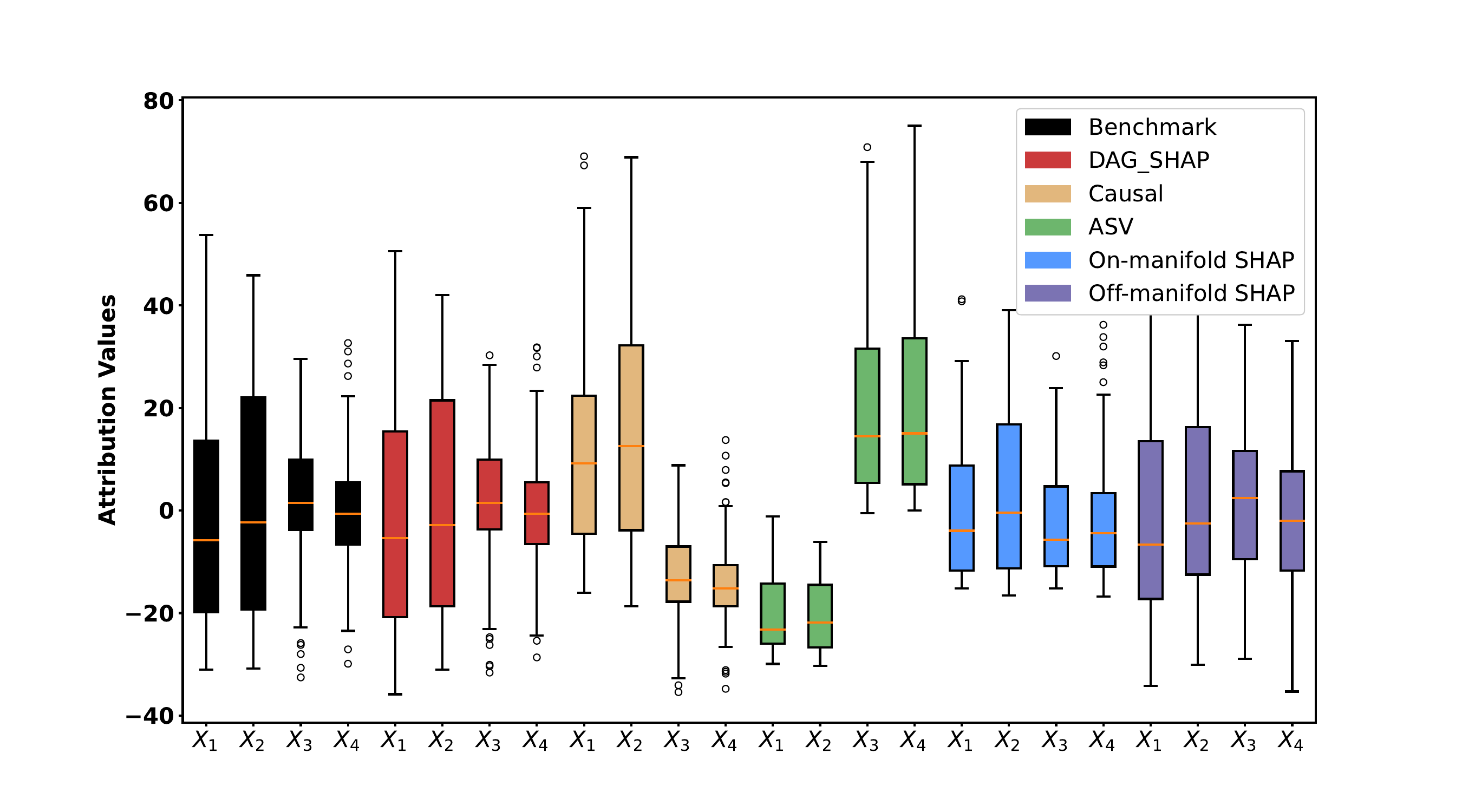}
    \caption{Attribution results of the synthetic dataset (DNN).}
    \label{fig:synthetic_results}
\end{figure}

\begin{table}[htbp]
\centering
\caption{Mean absolute error (MAE) on the synthetic dataset.}
\label{tab:mae_synthetic_results}
\sisetup{round-mode=places,round-precision=2}

\begin{tabular}{@{}l S[table-format=2.2] S[table-format=2.2] c@{}}
\toprule
\textbf{Algorithm} & {\textbf{MAE}} & {\textbf{$\times$ DAG-SHAP}} & \textbf{Rank} \\
\midrule
DAG-SHAP & 6.84  & 1.00  & 1 \\
Off-SHAP & 19.80 & 2.89  & 2 \\
On-SHAP  & 31.02 & 4.53  & 3 \\
Causal   & 58.75 & 8.59  & 4 \\
ASV      & 78.92 & 11.54 & 5 \\
\bottomrule
\end{tabular}
\end{table}

\partitle{Experimental results} ASV represents both the asymmetric Shapley value and the asymmetric causal Shapley value among the baseline algorithm, as they are equivalent given a DAG structure. To simplify, we use `Causal' to denote the symmetric causal Shapley value in the legend of Figure~\ref{fig:synthetic_results}. We omit the comparison with Shapley Flow and Recursive Shapley value because every feature is assumed to have an exogenous contribution in both the synthetic dataset and the real dataset used. The axis labels $X_1, X_2,X_3, X_4$ of benchmark correspond to the sum of the split nodes. Based on the data generation process, it is clear that the influence of $X_1$ and $X_2$ on $Y$ should be greater than that of $X_3$ and $X_4$. That is, the range of attribution values for features $X_1$ and $X_2$ should be greater than $X_3$ and $X_4$. However, the results from off-manifold SHAP, on-manifold SHAP, and asymmetry SHAP with a DNN contradict this as shown in Figure~\ref{fig:synthetic_results} , as they fail to identify the impact of source vertices on their child vertices. We also conduct a statistical analysis of the attribution error for each method.  The experimental results show that the DAG-SHAP values calculated on Figure~\ref{fig:DAG_synthetic} are approximately equal to the sum of the DAG-SHAP values calculated on Figure~\ref{fig:DAG_splitted_synthetic}, and are also approximately equal to the attribution values in benchmark. The error of other methods is significantly greater than that of DAG-SHAP as shown in Table~\ref{tab:mae_synthetic_results}. The attribution results using the XGBoost model on the synthetic dataset and the mean absolute error between the results and the benchmark are shown in Figure~\ref{fig:synthetic_xgboost_results} and Table~\ref{tab:synthetic_xgboost_errors}, respectively.

\begin{figure}[htbp]
    \includegraphics[width=0.48\textwidth, height=0.2\textheight]{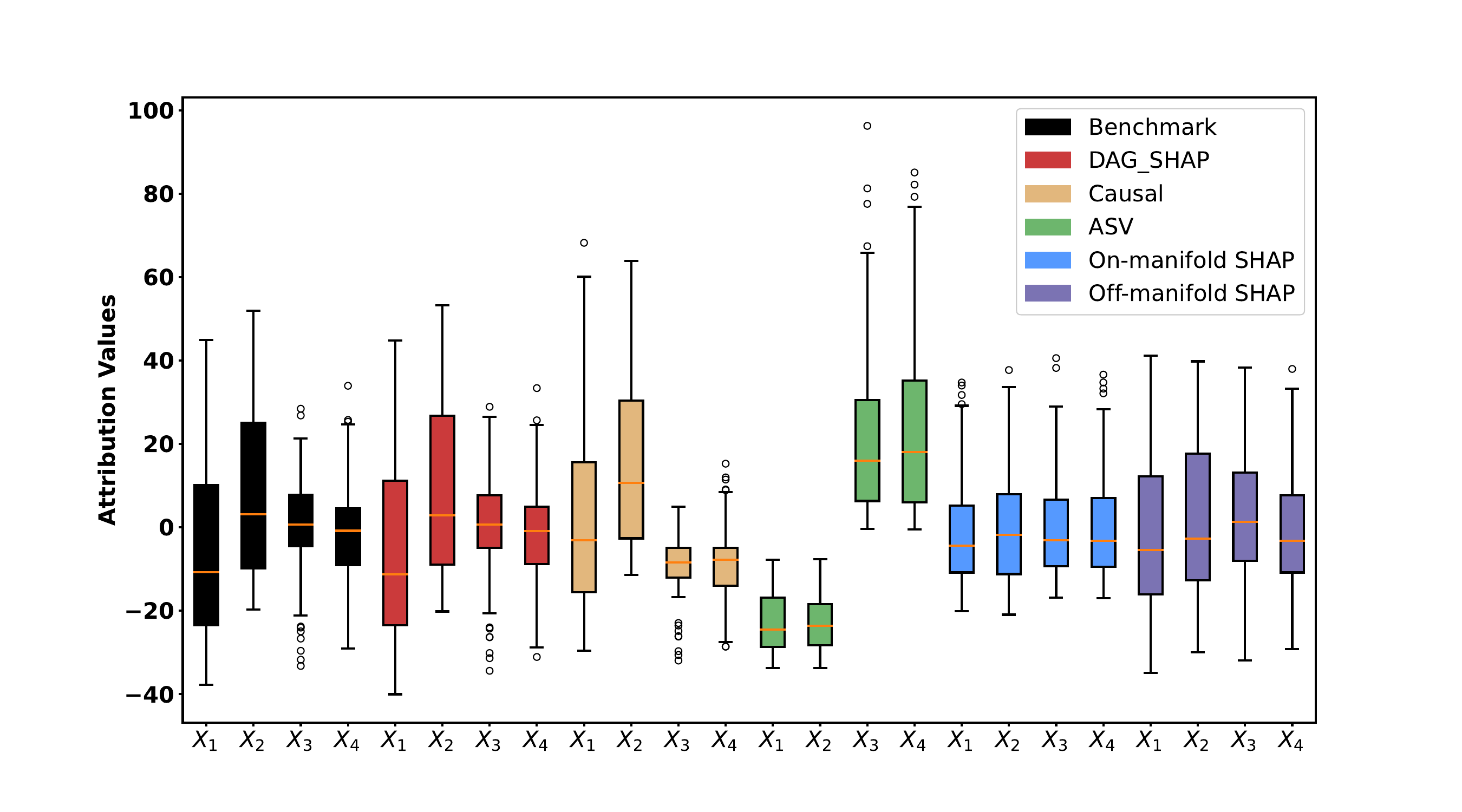}
        \caption{Attribution results of the synthetic dataset (XGBoost).}
        \label{fig:synthetic_xgboost_results}
\end{figure}


\begin{table}[htbp]
\centering
\caption{MAE on the synthetic dataset (XGBoost).}
\label{tab:synthetic_xgboost_errors}
\vspace{-0.5em}
\sisetup{round-mode=places,round-precision=2}
\begin{tabular}{@{}l S[table-format=1.2] S[table-format=2.2] c@{}}
\toprule
\textbf{Algorithm} & {\textbf{MAE}} & {\textbf{$\times$ DAG-SHAP}} & \textbf{Rank} \\
\midrule
DAG-SHAP & 4.21 & 1.00 & 1\\
Off-SHAP & 21.3 & 5.06 & 2\\
On-SHAP  & 32.15 & 7.64 & 3\\
Causal   & 33.98 & 8.07 & 4\\
ASV      & 94.09 & 22.35 & 5\\
\bottomrule
\end{tabular}
\end{table}

\subsection{Experiments on Real Datasets}

\begin{figure*}[htbp]
    \centering
    \includegraphics[width=\textwidth, height=0.2\textheight]{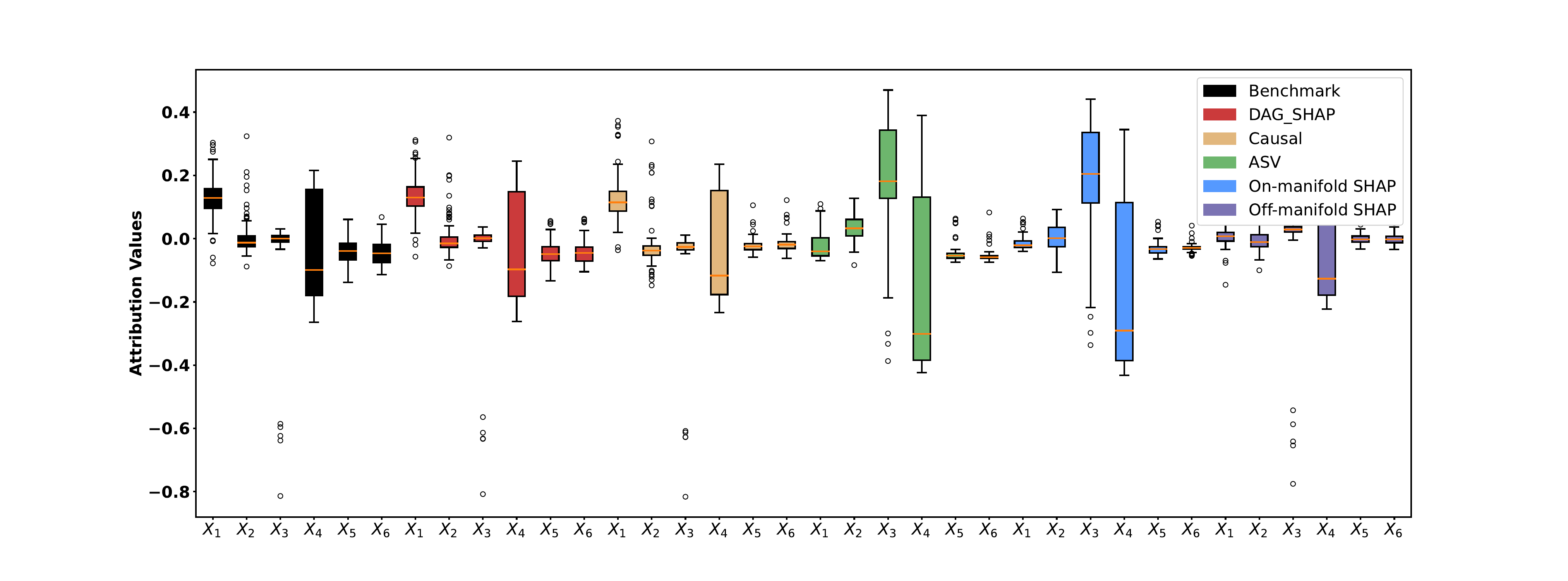}
    \vspace{-3em}
    \caption{Attribution results of Census dataset (DNN).}
    \label{fig:census_results}
\end{figure*}

\begin{table}[htbp]
\centering
\caption{MAE on the Census dataset (DNN).}
\label{tab:mae_census_results}
\sisetup{round-mode=places,round-precision=2}
\begin{tabular}{@{}l S[table-format=1.2] S[table-format=1.2] c@{}}
\toprule
\textbf{Algorithm} & {\textbf{MAE}} & {\textbf{$\times$ DAG-SHAP}} & \textbf{Rank} \\
\midrule
DAG-SHAP & 0.12 & 1.00 & 1 \\
Causal   & 0.21 & 1.75 & 2 \\
Off-SHAP & 0.29 & 2.42 & 3 \\
On-SHAP  & 0.61 & 5.08 & 4 \\
ASV      & 0.63 & 5.25 & 5 \\
\bottomrule
\end{tabular}
\end{table}

\begin{figure}[htbp]
    \centering
			\begin{tikzpicture}[node distance=1.3cm, on grid, auto]
                    \node[draw, ellipse,thick, minimum width=1.5cm, minimum height=0.75cm] (Country) at (0,-0.5) {$\textit{Country}$};
                    \node[draw, ellipse, thick, minimum width=1.5cm, minimum height=0.75cm] (Race) at (2,-0.5) {$\textit{Race}$};
                    \node[draw, ellipse,thick, minimum width=1.5cm, minimum height=0.75cm] (Age) at (4.1,-0.5) {$\textit{Age}$};
                    \node[draw, ellipse, thick, minimum width=1.5cm, minimum height=0.75cm] (Occ) at (0,-1.8) {$\textit{Occ}$};
                    \node[draw, ellipse, thick, minimum width=1.5cm, minimum height=0.75cm] (Capg) at (2,-1.8) {$\textit{Capg}$};
                    \node[draw, ellipse, thick, minimum width=1.5cm, minimum height=0.75cm] (MStat) at (4,-1.8) {$\textit{MS}$};
                    \node[draw, ellipse, thick, minimum width=1.5cm, minimum height=0.75cm] (Income) at (2,-3) {$\textit{Income}$};
            
                    \path[->, thick] 
                    (Country) edge (Occ)
                    (Country) edge (Capg)
                    (Country) edge (MStat)
                    (Race) edge (Occ)
                    (Race) edge (Capg)
                    (Race) edge (MStat)
                    (Age) edge (Occ)
                    (Age) edge (Capg)
                    (Age) edge (MStat)
                    (Occ) edge (Income)
                    (Capg) edge (Income)
                    (MStat) edge (Income);
                    \path[->, thick, bend left = 120] 
                    (Age) edge (Income);
                    \path[->,thick,bend left=60]
                    (Race) edge (Income);
                    \path[->,thick,bend right = 120]
                    (Country) edge (Income);
                \end{tikzpicture}
                \caption{DAG of the Census Income dataset.}\label{fig:DAG_census}
\end{figure}
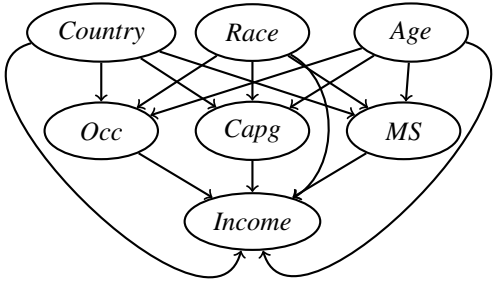


\partitle{Experiments on Census Income Dataset} The first real dataset we used is the Adult dataset~\cite{Dua:2019} with the causal graph shown in Figure~\ref{fig:DAG_census}. We train a binary classifier to predict whether the income of one individual exceeds \$50,000 per year. We also split the direct and indirect effects of Country, Race, and Age to create a new dataset. We use the same way in the above experiments to get benchmark attribution values. We show the attribution results with DNN in Figure~\ref{fig:census_results}, where $X_1, \cdots, X_6$ represent Race, Country, Age, Occupation, Marital Status, and Capital Gain, respectively. The experimental results show that the attribution value of DAG-SHAP is the closest to the benchmark, with the mean absolute error only $57.1\%$ of that of the second smallest method, causal Shapley value, as shown in Table~\ref{tab:mae_census_results}. The attribution results using the XGBoost model on the Census dataset and the mean absolute error between the results and the benchmark are shown in Figure~\ref{fig:census_xgboost_results} and Table~\ref{tab:census_xgboost_errors}, respectively.

\begin{figure*}[htbp]
    \centering
    \includegraphics[width=\textwidth, height=0.2\textheight]{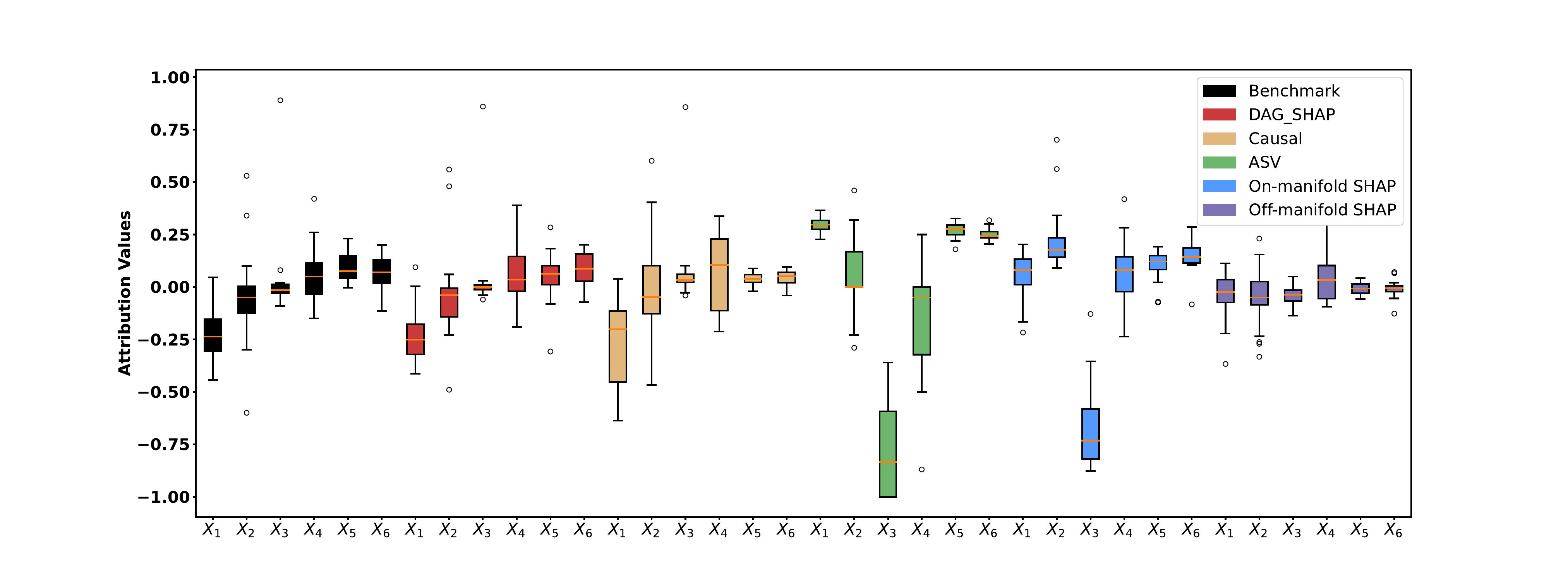}
        \caption{Attribution results of the Census dataset (XGBoost).}
        \label{fig:census_xgboost_results}
\end{figure*}

\begin{table}[htbp]
\centering
\caption{MAE on the Census dataset (XGBoost).}
\label{tab:census_xgboost_errors}
\sisetup{round-mode=places,round-precision=2}
\begin{tabular}{@{}l S[table-format=1.2] S[table-format=2.2] c@{}}
\toprule
\textbf{Algorithm} & {\textbf{MAE}} & {\textbf{$\times$ DAG-SHAP}} & \textbf{Rank} \\
\midrule
DAG-SHAP &  0.40 & 1.00 & 1\\
Causal   &  0.52 & 1.30 & 2\\
Off-SHAP &  0.59 & 1.48 & 3\\
On-SHAP  &  1.59 & 3.98 & 4\\
ASV      &  2.11 & 5.28 & 5\\
\bottomrule
\end{tabular}
\end{table}


\begin{figure}[htbp]
    \raggedright  
    \includegraphics[width=0.48\textwidth, height=0.2\textheight,trim=2cm 0 0 0, clip]{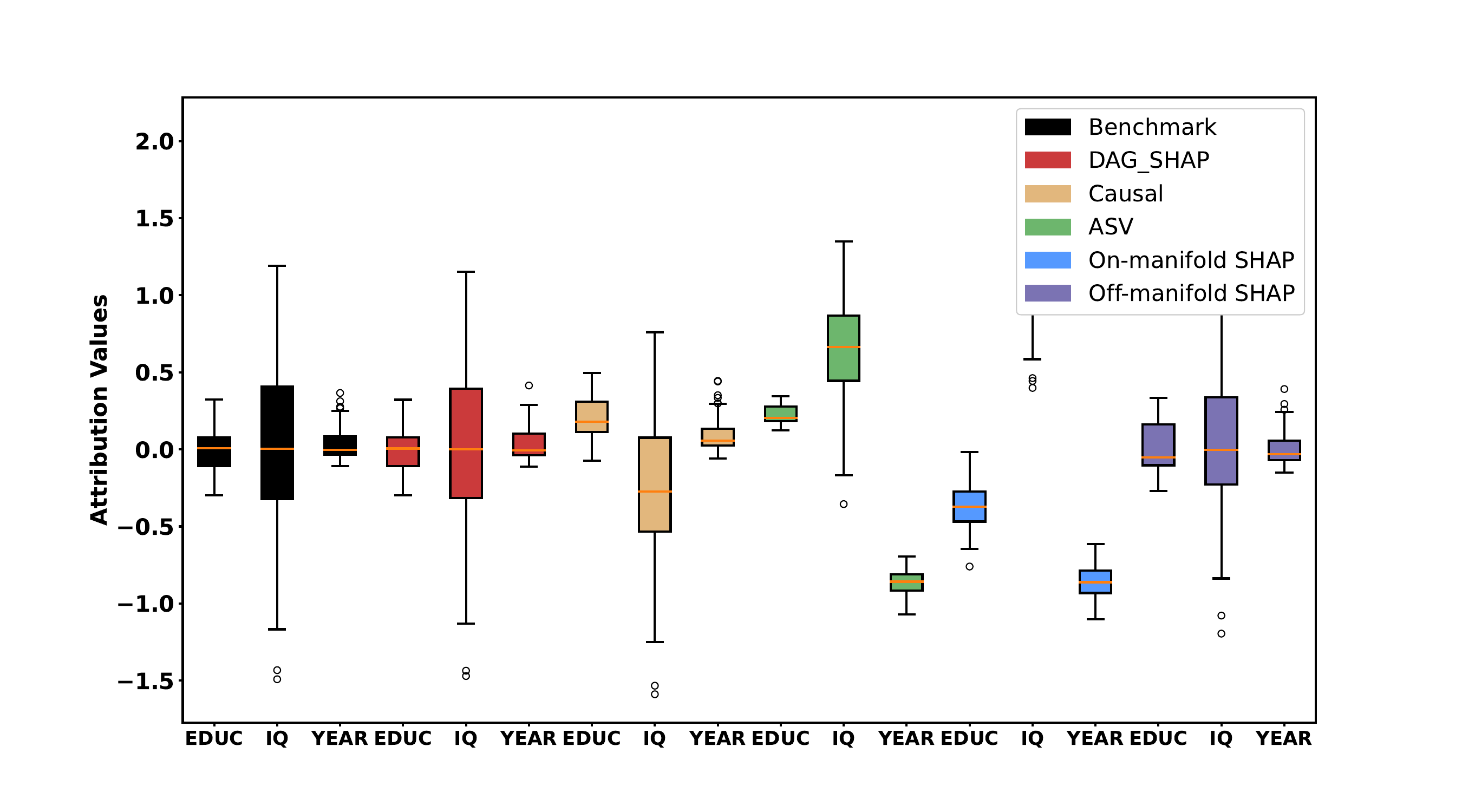}
    \caption{Attribution results of Griliches76 dataset.} \label{fig:griliches_results}
\end{figure}

\begin{table}[htbp]
\centering
\caption{MAE on the Griliches76 dataset (DNN).}
\label{tab:mae_griliches76_results}
\vspace{-0.5em}
\sisetup{round-mode=places,round-precision=2}
\begin{tabular}{@{}l S[table-format=1.2] S[table-format=2.2] c@{}}
\toprule
\textbf{Algorithm} & {\textbf{MAE}} & {\textbf{$\times$ DAG-SHAP}} & \textbf{Rank} \\
\midrule
DAG-SHAP & 0.08 & 1.00  & 1 \\
Off-SHAP & 0.27 & 3.38  & 2 \\
Causal   & 0.52 & 6.50  & 3 \\
ASV      & 1.76 & 22.00 & 4 \\
On-SHAP  & 2.44 & 30.50 & 5 \\
\bottomrule
\end{tabular}
\end{table}

\begin{figure}[h]
    \subfigure[Griliches76.]{
        \begin{minipage}[b]{0.20\textwidth}
        \begin{tikzpicture}[ node distance=1.3cm, on grid, auto]
            \node[draw, circle, thick] (IQ) {$IQ$};
            \node[draw, circle, thick] (EDU) [below right=of IQ] {$EDU$};
            \node[draw, circle,thick] (YEAR) [below left=of IQ] {$YEAR$};
            \node[draw, circle, thick] (LW) [below right=of YEAR] {$LW$};
            
            \path[->,thick] 
            (IQ) edge (EDU)
            (EDU) edge (LW)
            (YEAR) edge (LW)
            (IQ) edge (LW);
        \end{tikzpicture}
        \label{fig:DAG_griliches76}
    \end{minipage}%
    }
    \hspace{0.01\textwidth}
    \subfigure[Splitted Griliches76.]{
        \begin{minipage}[b]{0.22\textwidth}
        \begin{tikzpicture}[ node distance=1.3cm, on grid, auto]
            \node[draw, circle,thick,inner sep=2pt] (IQtilde) {$\tilde{IQ}$};
            \node[draw, circle,thick,inner sep=2pt] (EDU) [below right=of IQtilde] {$EDU$};
            \node[draw, circle,thick,inner sep=2pt] (YEAR) [below left=of IQtilde] {$YEAR$};
            \node[draw, circle,thick,inner sep=2pt] (LW) [below right=of YEAR] {$LW$};
            \node[draw, circle,thick] (IQbar) [right=of IQtilde,xshift=0.3cm] {$\bar{IQ}$};
            
            \path[->,thick] 
            (IQtilde) edge (LW)
            (EDU) edge (LW)
            (YEAR) edge (LW)
            (IQbar) edge (EDU);
        \end{tikzpicture}
        \label{fig:DAG_splitted_griliches76}
    \end{minipage}
    }

    \caption{DAGs of the Griliches76 datasets.}
    \vspace{-1em}
\end{figure}

\partitle{Experiments on Griliches76 Dataset} The second real dataset we used is the Griliches76 dataset~\cite{griliches1976wages}, gathered from the U.S. labor market. This dataset is widely used in research to explore the impact of features on income. We selected three features: IQ, education level (EDU), and years working at the current unit (YEAR). The target variable is the logarithm of weekly income (LW). We use the same way to create the benchmark as the synthetic dataset. The original causal graph is shown in Figure~\ref{fig:DAG_griliches76} and the split graph is shown in Figure~\ref{fig:DAG_splitted_griliches76}. We train an ML model with input features IQ, EDU, and YEAR. The attribution results with a DNN of each method are shown in Figure~\ref{fig:griliches_results}. The experimental results show that the attribution value of DAG-SHAP is the closest to the benchmark, with the MAE only $29.6\%$ of that of the second smallest method, Off-manifold SHAP, as shown in Table~\ref{tab:mae_griliches76_results}. The attribution results using the XGBoost model on the Griliches76 dataset and the mean absolute error between the results and the benchmark are shown in Figure~\ref{fig:griliches_xgboost_results} and Table~\ref{tab:griliches_xgboost_errors}, respectively.

\begin{figure}[htbp]
   \includegraphics[width=0.48\textwidth, height=0.2\textheight]{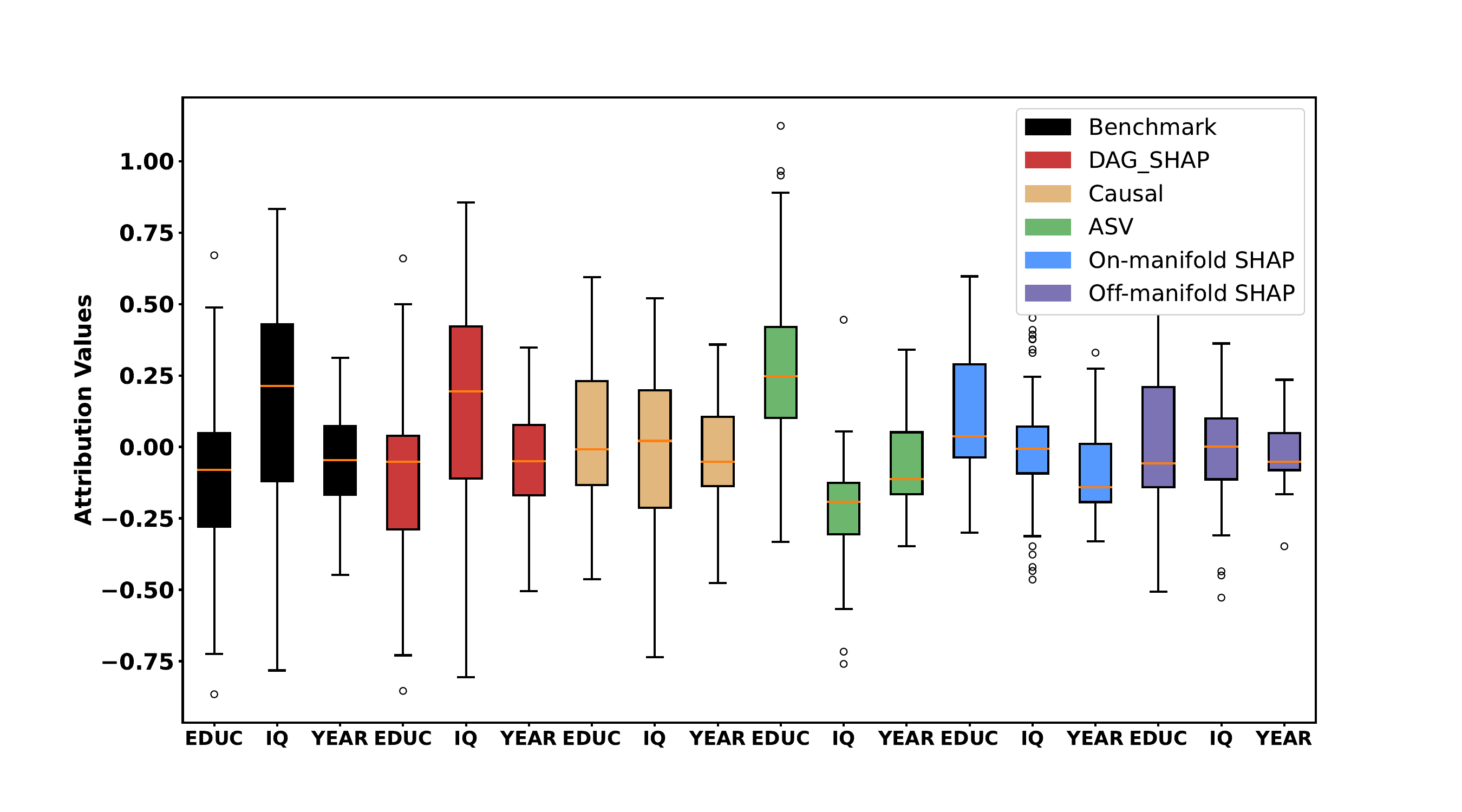}
        \caption{Attribution results of the Griliches76 dataset (XGBoost).}
        \label{fig:griliches_xgboost_results}
\end{figure}


\begin{table}[htbp]
\centering
\caption{MAE on the Griliches76 dataset (XGBoost).}
\label{tab:griliches_xgboost_errors}
\vspace{-0.5em}
\sisetup{round-mode=places,round-precision=2}
\begin{tabular}{@{}l S[table-format=1.2] S[table-format=2.2] c@{}}
\toprule
\textbf{Algorithm} & {\textbf{MAE}} & {\textbf{$\times$ DAG-SHAP}} & \textbf{Rank} \\
\midrule
DAG-SHAP & 0.05 & 1.00 & 1\\
Causal   & 0.41 & 8.20 & 2\\
Off-SHAP & 0.58 & 11.60 & 3\\
On-SHAP  & 0.67 & 13.40 & 4\\
ASV      & 0.95 & 19.00 & 5\\
\bottomrule
\end{tabular}
\end{table}

\subsection{Approximation Error Evaluation} \label{sec:efficiencyexp}
The attribution values in our experiments are estimated by approximate sampling methods. Therefore, the MAE between an attribution method and the benchmark may include both the intrinsic discrepancy of the method and the error introduced by sampling. To examine whether the observed superiority of DAG-SHAP is mainly caused by sampling variance, we conduct an additional repeated-sampling experiment.

In the above experiments, 200 permutations are used for sampling when attributing each explained data point. We independently repeat the attribution process twice for each explained data point, with each attribution method still using 200 permutations, and compute the MAE between the two independently estimated attribution results. This repeated-sampling MAE is regarded as an estimate of the approximation error. Let \(E_{\mathrm{bench}}^{d,m}\) denote the MAE between method \(m\) and the benchmark under dataset/model setting \(d\), and let \(E_{\mathrm{approx}}^{d,m}\) denote the corresponding repeated-sampling approximation error. We report the adjusted MAE as
\[
E_{\mathrm{adj}}^{d,m}
=
\bigl|E_{\mathrm{bench}}^{d,m}-E_{\mathrm{approx}}^{d,m}\bigr|.
\]
The adjusted MAE values are shown in Table~\ref{table:MSE}. The results indicate that, even after excluding the estimated approximation error, DAG-SHAP still has the smallest adjusted MAE across all datasets and model architectures. This suggests that the advantage of DAG-SHAP mainly comes from its attribution mechanism rather than from sampling approximation effects.

\begin{table*}[t]
\centering
\small
\setlength{\tabcolsep}{4pt}
\caption{Adjusted MAE after excluding the estimated approximation error.}
\label{table:MSE}
\begin{tabular}{@{}lcccccc@{}}
\toprule
\textbf{Algorithm}
& \textbf{Synthetic-DNN}
& \textbf{Griliches-DNN}
& \textbf{Census-DNN}
& \textbf{Synthetic-XGB}
& \textbf{Griliches-XGB}
& \textbf{Census-XGB} \\
\midrule
DAG-SHAP & 3.123  & 0.048 & 0.038 & 0.297  & 0.015 & 0.267 \\
Causal   & 56.829 & 0.495 & 0.146 & 31.083 & 0.390 & 0.431 \\
ASV      & 75.985 & 1.743 & 0.555 & 90.526 & 0.915 & 1.988 \\
On-SHAP  & 27.698 & 2.419 & 0.554 & 28.809 & 0.641 & 1.519 \\
Off-SHAP & 16.812 & 0.254 & 0.202 & 18.374 & 0.562 & 0.492 \\
\bottomrule
\end{tabular}
\end{table*}

\subsection{Efficiency Analysis} \label{sec:scalabilityexp}
When dealing with large datasets, parallel computing can effectively accelerate feature attribution calculations. Feature attribution for different data points can be executed in parallel. Additionally, within the feature attribution process for a single data point, sampling different permutations can also be executed in parallel. We extend the synthetic dataset in Section~\ref{sec:syntheticexp} such that each data point has 100 features and the DAG contains 200 edges. Specifically, we replicate $X_1$, $X_2$, $X_3$, and $X_4$ from the synthetic dataset 25 times, resulting in 100 features in total. Each replicated feature maintains its causal relationship with $Y$, ensuring that the collective contribution of all features to $Y$ remains consistent with the contribution in the original structure. We use the same neural network structure as in Section~\ref{sec:syntheticexp} and conduct experiments on a server equipped with two AMD EPYC 9754 128-Core Processors, providing a total of 512 logical processors. For each data point, we conduct two independent rounds of sampling and compare the mean absolute error (MAE) between the sampling results. The experimental results show that with 128 sampling permutations per data point, the MAE is 5.17\%, which is significantly smaller than the errors caused by different feature attribution algorithms. For instance, the absolute error between Off-SHAP, which has the smallest MAE among the baselines, and the benchmark is 17.24\%. Using 128 threads for parallel sampling of 128 permutations, the computation time is only 57.5 seconds. Additionally, we conduct experiments with 256 and 384 permutations, resulting in mean absolute errors of 4.23\% and 3.71\%, respectively.

\section{Conclusion}\label{sec:conclusion}
In this paper, we propose an innovative feature attribution method that incorporates causal relationships within Directed Acyclic Graphs (DAGs). We introduce an edge intervention method that targets specific child vertices via their parent vertices to control the fine-grained causal influence. Our method uniquely addresses the limitations of traditional feature attribution techniques by capturing both the exogenous contributions of edges and the requisite externality contributions simultaneously. Extensive experiments on both real and synthetic datasets validate the effectiveness of DAG-SHAP. These results highlight the value of edge-level explanations and suggest extending DAG-SHAP to time-dependent settings.


\newpage

\IEEEtriggeratref{44}
\bibliographystyle{abbrv}
\bibliography{refs}


\onecolumn
\section*{Appendix}

\renewcommand{\thesubsection}{\Alph{section}.\arabic{subsection}}

In the appendix of our paper, we provide comprehensive additional content. Section~\ref{sec:mostrealted} introduces the Shapley value-based feature attribution methods considered in this paper. In Section~\ref{sec:appendix_properties}, we discuss why DAG-SHAP satisfies all the desirable properties. Section~\ref{sec:comparasion_asy_Intervention} presents the comparison of different intervention methods.



\section{Shapley Value-Based Feature Attribution Methods}\label{sec:mostrealted}

In this section, we provide a detailed introduction to the baselines and the comparative methods used in our study. 

\textbf{On-manifold Shapley} value was proposed by~\cite{scott2017unified}, and it is defined as follows
\begin{align*}
\phi_i=\sum_{j=0}^{n-1} \frac{1}{n\binom{n-1}{j}} \sum_{\substack{i \notin \mathcal{S}, |\mathcal{S}|=j}}\{ \mathbb{E}\left[f\left(\vx_{\mathcal{S} \cup\{i\}}, \rvx_{\overline{\mathcal{S} \cup\{i\}}}\right) \mid \rvx_{\mathcal{S} \cup\{i\}}=\vx_{\mathcal{S} \cup\{i\}}\right]-\mathbb{E}\left[f\left(\vx_\mathcal{S}, \rvx_{\bar{\mathcal{S}}}\right) \mid \rvx_\mathcal{S}=\vx_\mathcal{S} \right]\}.
\end{align*}
The marginal contribution of feature $i$ in the explained input $\vx$ when cooperating with a coalition $\mathcal{S}$ is expressed as 
$$\mathbb{E}\left[f\left(\vx_{\mathcal{S} \cup\{i\}}, \rvx_{\overline{\mathcal{S} \cup\{i\}}}\right) \mid \rvx_{\mathcal{S} \cup\{i\}}=\vx_{\mathcal{S} \cup\{i\}}\right]-\mathbb{E}\left[f\left(\vx_\mathcal{S}, \rvx_{\bar{\mathcal{S}}}\right) \mid \rvx_\mathcal{S}=\vx_\mathcal{S} \right],$$ where $\overline{\mathcal{S}\cup\{i\}}$ represents the complement of $\mathcal{S}\cup\{i\}$. The values of features not in the coalitions are determined under conditional expectations of features in the coalitions.

\textbf{Off-manifold Shapley} value~\cite{scott2017unified} is a simplified version of on-manifold Shapley value that assumes feature independence for implementation, and it is widely adopted in practical applications. The marginal contribution in off-manifold Shapley value is defined as $\mathbb{E}\left[f\left(\vx_{\mathcal{S} \cup\{i\}}, \rvx_{\overline{\mathcal{S} \cup\{i\}}}\right)\right]-\mathbb{E}\left[f\left(\vx_\mathcal{S}, \rvx_{\bar{\mathcal{S}}}\right)\right]$. The values of features $\overline{\mathcal{S}\cup\{i\}}$ and $\overline{\mathcal{S}}$ which are not in the selected coalitions are determined by the average values in the dataset. 

\textbf{Asymmetry Shapley} value~\cite{frye2020asymmetric} sets the weight of permutations that do not satisfy the topological order to zero, thereby breaking the symmetry in Shapley value. This allows for model explainability analysis to distinguish the causal order between features when considering their contributions. It accounts for the causal relationships between features, rather than assuming that all features affect the model's output symmetrically and independently.

\textbf{Causal Shapley} value~\cite{DBLP:conf/nips/HeskesSBC20} incorporates a priori causal knowledge, using interventions in the causal domain to represent the interaction influence between features. The marginal contribution is expressed as $\mathbb{E}\left[f\left(\vx_{\mathcal{S} \cup\{i\}}, \rvx_{\overline{\mathcal{S} \cup\{i\}}}\right) \mid \operatorname{do}\left(\rvx_{\mathcal{S} \cup\{i\}}=\vx_{\mathcal{S} \cup\{i\}}\right)\right]-\mathbb{E}\left[f\left(\vx_\mathcal{S}, \rvx_{\bar{\mathcal{S}}}\right) \mid \operatorname{do}\left(\rvx_\mathcal{S}=\vx_\mathcal{S}\right)\right]$, where operator $\operatorname{do}(\cdot)$ means setting a variable to a specific value. The intervention is conducted on the feature vertex so the influence of the parent vertex will transfer to all its child vertices simultaneously if it is intervened. Note that the causal Shapley value can adopt both symmetric and asymmetric permutations, referred to as symmetry causal Shapley value and asymmetry causal Shapley value, respectively. Asymmetry causal Shapley value is equivalent to the asymmetry Shapley value when DAG is known.

\textbf{Shapley Flow}~\cite{DBLP:conf/aistats/WangWL21} is a method that attributes based on paths, ensuring path order consistency through depth-first search. It sums the attribution values of all paths passing through an edge as the edge's attribution, ensuring that the sum of the attribution values of all edges in any cut in the causal graph is the same. Formally, the Shapley value of path $i$ can be represented as follows
\begin{equation} \nonumber
    \tilde{\phi}_\nu(i)=\sum_{\pi \in \Pi_{\mathrm{dfs}}} \frac{\tilde{v}([j: \pi(j) \leq \pi(i)])-\tilde{v}([j: \pi(j)<\pi(i)])}{\left|\Pi_{\mathrm{dfs}}\right|},
\end{equation}
where inequality $\pi(j) \leq \pi(i)$ denotes that path $j$ precedes path $i$ under ordering $\pi \in \Pi_{\mathrm{dfs}}$. To obtain the attribution of an edge $e$, they propose to sum the attributions of all paths $\sP$ that contains $e$, i.e., $\phi_v(e)=\sum_{p \in \sP, e \in p} \tilde{\phi}_v(p)$.

\textbf{Recursive Shapley} value~\cite{DBLP:conf/icml/SingalMN21} conducts a top-down feature attribution process to quantify how changes at the source vertices propagate through the graph. It assumes that only the vertices with no incoming edges have exogenous contributions, and the relationship between intermediate vertices and the source vertices is deterministic, with intermediate vertices only transmitting the effects of source vertices.

\textbf{Shapley-ICC}~\cite{janzing2024quantifying}. Denote the worth of a coalition of noise terms $N_T$ for $T \subset \{1, \dots, n\}$ by $\nu(T) := -\psi(X_n|N_T)$. Then the (Shapley based) ICC of each node $X_j$ to the uncertainty of $X_n$ reads:
\begin{align}
    ICC^{Sh}_\psi (X_j \rightarrow X_n) &:= \sum_{T \subseteq U \setminus \{j\}} \frac{1}{n\binom{n-1}{|T|}} [ \nu(T \cup \{j\}) - \nu(T)] \\
    &= \sum_{T \subseteq U \setminus \{j\}} \frac{1}{n\binom{n-1}{|T|}} ICC_\psi (X_j \rightarrow X_n | T),
\end{align}
where $U := \{1, \dots, n\}$. $\psi(X_n|N_T)$ can be computed by conditional Shannon entropy or expected conditional variance with a prerequisite that the noise term of each feature can be isolated from parents based on the operations of the structural causal model.

\section{Theoretical Analysis}\label{sec:appendix_properties}
In this section, we prove that DAG-SHAP satisfies the attribution properties defined in Section~\ref{sec:fundamental_properties}.
Although these properties are stated at the feature level in the preliminaries, DAG-SHAP first assigns attributions to edges and then aggregates them to features by $\Phi(k)=\sum_{e\in\mathcal{O}_k}\Psi(e)$.

\subsection{Why DAG-SHAP Satisfies These Properties}
\begin{proof}[Proof of Linearity]
    For any edge $e_i$, its edge intervention causal Shapley value is given by
    \begin{equation}
        \Psi_{f}(e_i) = \sum_{\pi \in \Pi} \frac{1}{|\Pi|}\{\mathbb{E}[f(\rvx)|\operatorname{do}(\rve_{\underline{S}_{\pi}^i}=e_{\underline{\mathcal{S}}_{\pi}^i})]-\mathbb{E}[f(\rvx)|\operatorname{do}(\rve_{\mathcal{S}_{\pi}^i}=e_{\mathcal{S}_{\pi}^i})]\}.
    \end{equation}
    Substituting $f = f_u+f_w$, we can get
    \begin{equation}
        \Psi_{f}(e_i) = \sum_{\pi \in \Pi} \frac{1}{|\Pi|}\{\mathbb{E}[f_u(\rvx)+f_w(\rvx)|\operatorname{do}(\rve_{\underline{S}_{\pi}^i}=e_{\underline{\mathcal{S}}_{\pi}^i})]-\mathbb{E}[f_u(\rvx)+f_w(\rvx)|\operatorname{do}(\rve_{\mathcal{S}_{\pi}^i}=e_{\mathcal{S}_{\pi}^i})]\}.
    \end{equation}
    This can be split into two parts:
    \begin{align}
        \Psi_{f}(e_i)  = & \sum_{\pi \in \Pi} \frac{1}{|\Pi|}\{\mathbb{E}[f_u(\bm{x})|\operatorname{do}(\rve_{\underline{S}_{\pi}^i}=e_{\underline{\mathcal{S}}_{\pi}^i})]-\mathbb{E}[f_u(\bm{x})|\operatorname{do}(\rve_{\mathcal{S}_{\pi}^i}=e_{\mathcal{S}_{\pi}^i})]\} \\
        & +\sum_{\pi \in \Pi} \frac{1}{|\Pi|}\{\mathbb{E}[f_w(\bm{x})|\operatorname{do}(\rve_{\underline{S}_{\pi}^i}=e_{\underline{\mathcal{S}}_{\pi}^i})]-\mathbb{E}[f_w(\bm{x})|\operatorname{do}(\rve_{\mathcal{S}_{\pi}^i}=e_{\mathcal{S}_{\pi}^i})]\}.
    \end{align}
    According to the definition of edge intervention causal Shapley values, the two summations correspond to $\Psi_{f_u}(e_i)$ and $\Psi_{f_w}(e_i)$, respectively. Thus, $\Psi_{f}(e_i) = \Psi_{f_u}(e_i) + \Psi_{f_w}(e_i)$. This proves that  DAG-SHAP satisfies the Linearity property. The proof can be easily extended with respect to the feature node, since the attribution of a feature is the sum of the attribution values of its outgoing edges.
\end{proof}

\begin{proof}[Proof of Implementation Invariance]
If $f(\bm{x}) = g(\bm{x})$ for all inputs $\bm{x}$, the marginal contributions of edge $e_i$ in $f$ and $g$ with any edge subsets $e_{\mathcal{S}_{\pi}^i}$ are equal: $\mathbb{E}[f(\bm{x})|\operatorname{do}(\rve_{\underline{S}_{\pi}^i}=e_{\underline{\mathcal{S}}_{\pi}^i})]-\mathbb{E}[f(\bm{x})|\operatorname{do}(\rve_{\mathcal{S}_{\pi}^i}=e_{\mathcal{S}_{\pi}^i})] = \mathbb{E}[g(\bm{x})|\operatorname{do}(\rve_{\underline{S}_{\pi}^i}=e_{\underline{\mathcal{S}}_{\pi}^i})]-\mathbb{E}[g(\bm{x})|\operatorname{do}(\rve_{\mathcal{S}_{\pi}^i}=e_{\mathcal{S}_{\pi}^i})].$
This follows directly from the fact that $f$ and $g$ produce identical outputs for any subset of features.
The edge intervention causal Shapley value for an edge $i$ is a weighted sum of the marginal contributions across all  permutations within \(\Pi\):
\begin{equation}
 \Psi_f(e_i) = \sum_{\pi \in \Pi} \frac{1}{|\Pi|}\{\mathbb{E}[f(\rvx)|\operatorname{do}(\re_{{\underline{\mathcal{S}}}_{\pi}^i}=e_{{\underline{\mathcal{S}}}_{\pi}^i })]-\mathbb{E}[f(\rvx)|\operatorname{do}(\re_{\mathcal{S}_{\pi}^i}=e_{\mathcal{S}_{\pi}^i})]\}.  
\end{equation}
Substituting the equivalence of marginal contributions:
\begin{equation}
 \Psi_g(e_i) = \sum_{\pi \in \Pi} \frac{1}{|\Pi|}\{\mathbb{E}[g(\rvx)|\operatorname{do}(\re_{{\underline{\mathcal{S}}}_{\pi}^i}=e_{{\underline{\mathcal{S}}}_{\pi}^i })]-\mathbb{E}[g(\rvx)|\operatorname{do}(\re_{\mathcal{S}_{\pi}^i}=e_{\mathcal{S}_{\pi}^i})]\}.  
\end{equation}
Thus, the edge intervention causal Shapley for $f$ and $g$ are identical for all edges $e_i$ when $f(\bm{x}) = g(\bm{x})$ for all $\bm{x}$. This proves that  DAG-SHAP satisfies the Implementation Invariance property. The proof can be easily extended with respect to the feature node too, since the attribution of a feature is the sum of the attribution values of its outgoing edges.
\end{proof}

\begin{proof}[Proof of Sensitivity]
\( \vx \) and \( \vx' \) differ only in one feature, say $k$, and the target label $Y$ for the inputs is different. This implies there is a direct edge between the feature $k$ and the target label, denoted as $e_i$.  Since the feature does not influence other features, no other edges are involved. In each permutation \( \pi \), the marginal contribution of edge \( e_i \) is:
\[
   \Delta_{\pi}^i = \mathbb{E}[f(\rvx) \mid \operatorname{do}(\re_{\underline{\mathcal{S}}_{\pi}^i} = e_{\underline{\mathcal{S}}_{\pi}^i})] - \mathbb{E}[f(\rvx) \mid \operatorname{do}(\re_{\mathcal{S}_{\pi}^i} = e_{\mathcal{S}_{\pi}^i})].
\]
Since $\vx$ and $\vx'$ differ only in $e_i$, and $f(\vx) \neq f(\vx')$, changing $e_i$ while keeping other edges fixed leads to a change in the prediction. Therefore, the marginal contribution satisfies $\Delta_{\pi}^i \neq 0$.

Because the attribution of $e_i$ is non-zero, and the attribution values of all outgoing edges of $k$ (other than $e_i$) are zero (as interventions on these edges from $\vx'$ to $\vx$ do not change the prediction), DAG-SHAP assigns a non-zero attribution to feature $k$. This satisfies the Sensitivity property.
\end{proof}

\begin{proof}[Proof of Dummy]
Since feature \(k\) has no exogenous influence on the outcome, intervening on its outgoing edges does not affect the expected model output. Therefore, for any subset \( \mathcal{S} \subseteq E \setminus \{ e_i \} \) where $i\in\mathcal{O}_k$:
\[
\mathbb{E}\left[ f(\mathbf{x}) \mid \operatorname{do}\left( \re_{\mathcal{S} \cup \{ e_i \}} = e_{\mathcal{S} \cup \{ e_i \}} \right) \right] = \mathbb{E}\left[ f(\mathbf{x}) \mid \operatorname{do}\left( \re_{\mathcal{S}} = e_{\mathcal{S}} \right) \right].
\]
This implies that the marginal contribution of any outgoing edge \( e_i \) is zero:
\[
\Delta_{\pi}^i = \mathbb{E}\left[ f(\mathbf{x}) \mid \operatorname{do}\left( \re_{\underline{\mathcal{S}}_{\pi}^{i}} = e_{\underline{\mathcal{S}}_{\pi}^{i}} \right) \right] - \mathbb{E}\left[ f(\mathbf{x}) \mid \operatorname{do}\left( \re_{\mathcal{S}_{\pi}^{i}} = e_{\mathcal{S}_{\pi}^{i}} \right) \right] = 0.
\]
Since the marginal contributions \( \Delta_{\pi}^{i} \) are zero for all permutations \( \pi \) and all outgoing edges \( e_i \), the DAG-SHAP attribution for each edge is:
\[
\Psi(e_i) = \frac{1}{|\Pi|} \sum_{\pi \in \Pi} \Delta_{\pi}^{i} = 0.
\]
The total attribution for feature \(k\) is the sum of the attributions of its outgoing edges:
\[
\Phi(k) = \sum_{i\in\mathcal{O}_k} \Psi(e_i) = \sum_{i\in\mathcal{O}_k} 0 = 0.
\]
Since all outgoing edges of feature \( k \) have zero attributions and the feature's attribution is the sum of these, the feature's total attribution is zero. Therefore, DAG-SHAP satisfies the Dummy property.
\end{proof}

\begin{proof}[Proof of Causality]
In any valid topological order \(\pi \in \Pi\), all edges respect the causal constraint: if \(e_j \to e_k\), then \(e_j\) appears before \(e_k\) (\(\pi(j) < \pi(k)\)). Thus, when calculating \(\Psi(e_i)\), all edges in \(\mathcal{S}_{\pi}^i\) have already been intervened upon, ensuring that the causal effects of ancestor nodes are incorporated correctly, avoiding any reversal of causality. The intervention \(\operatorname{do}(\re_{{\underline{\mathcal{S}}}_{\pi}^i}=e_{{\underline{\mathcal{S}}}_{\pi}^i})\) ensures that the model output depends only on the current edge \(e_i\) and the previously intervened edges in \(\mathcal{S}_{\pi}^i\).
By computing the difference in expected model output, the calculation isolates the direct causal effect of \(e_i\), excluding any indirect effects through other paths.
Let the edge \(e_i = (p_i, c_i, \vx_{p_i})\), and its attribution value is defined as:
\[
 \Psi(e_i) = \frac{1}{|\Pi|} \sum_{\pi \in \Pi} \Delta_{\pi}^i,
\]
where \(\Delta_{\pi}^i = \mathbb{E}[f(\rvx)|\operatorname{do}(\re_{{\underline{\mathcal{S}}}_{\pi}^i}=e_{{\underline{\mathcal{S}}}_{\pi}^i})] - \mathbb{E}[f(\rvx)|\operatorname{do}(\re_{\mathcal{S}_{\pi}^i}=e_{\mathcal{S}_{\pi}^i})]\).  
\(\Delta_{\pi}^i\) depends solely on the effect of the edge \(e_i\) under the current intervention, ensuring that it measures only the direct causal contribution of \(e_i\).
DAG-SHAP's attribution process strictly respects causal ordering and isolates the effects of individual edges through edge interventions. The resulting attribution values reflect the direct causal impact of features on the output, thus satisfying the causality property.
\end{proof}

\begin{proof}[Proof of Efficiency]
The sum of attribution values over all edges is:
\[\sum_{e_i \in E} \Psi(e_i) = \sum_{e_i \in E} \sum_{\pi \in \Pi} \frac{1}{|\Pi|} \left[\mathbb{E}[f(\rvx) \mid \operatorname{do}(\re_{{\underline{\mathcal{S}}}_{\pi}^i} = e_{{\underline{\mathcal{S}}}_{\pi}^i})] - \mathbb{E}[f(\rvx) \mid \operatorname{do}(\re_{\mathcal{S}_{\pi}^i} = e_{\mathcal{S}_{\pi}^i})]\right].
\]
Switching the order of summation over \(e_i\) and \(\pi\):
\[\sum_{e_i \in E} \Psi(e_i) = \sum_{\pi \in \Pi} \frac{1}{|\Pi|} \sum_{e_i \in E} \left[\mathbb{E}[f(\rvx) \mid \operatorname{do}(\re_{{\underline{\mathcal{S}}}_{\pi}^i} = e_{{\underline{\mathcal{S}}}_{\pi}^i})] - \mathbb{E}[f(\rvx) \mid \operatorname{do}(\re_{\mathcal{S}_{\pi}^i} = e_{\mathcal{S}_{\pi}^i})]\right].\]
For a given permutation \(\pi\), summing over all edges \(e_i \in E\) creates a telescoping series:
\[
\sum_{e_i \in E} \left[\mathbb{E}[f(\rvx) \mid \operatorname{do}(\re_{{\underline{\mathcal{S}}}_{\pi}^i} = e_{{\underline{\mathcal{S}}}_{\pi}^i})] - \mathbb{E}[f(\rvx) \mid \operatorname{do}(\re_{\mathcal{S}_{\pi}^i} = e_{\mathcal{S}_{\pi}^i})]\right] = \mathbb{E}[f(\rvx) \mid \operatorname{do}(\re_E = e_E)] - \mathbb{E}[f(\rvx)],
\]
where \(\mathbb{E}[f(\rvx) \mid \operatorname{do}(\re_E = e_E)]\) is the model output when all edges are intervened, and \(\mathbb{E}[f(\rvx)]\) is the baseline output when no edges are intervened.
Averaging over all permutations \(\pi \in \Pi\) does not affect the telescoping result because the final and initial terms are the same for every permutation:
\[\sum_{\pi \in \Pi} \frac{1}{|\Pi|} \left[\mathbb{E}[f(\rvx) \mid \operatorname{do}(\re_E = e_E)] - \mathbb{E}[f(\rvx)]\right] = \mathbb{E}[f(\rvx) \mid \operatorname{do}(\re_E = e_E)] - \mathbb{E}[f(\rvx)].
   \]
Thus, the sum of the attribution values for all edges satisfies:
\[\sum_{e_i \in E} \Psi(e_i) = f(\vx) - \mathbb{E}[f(\rvx)].\]
This proves that DAG-SHAP satisfies the Efficiency property because the sum of the attribution values of all points is equal to the sum of the attribution values of all edges. 
\end{proof}

\begin{proof}[Proof of Externality]  
Consider two vertices \(k\) and \(k^\prime\): \(k^\prime\) has a path to \(y\), forming a path \(k^\prime \to y\),  \(k\) has a path \(k \to y\) independent of \(k^\prime\).  When computing the marginal contribution of edges associated with \(k\), the effects of \(k^\prime \to y\) are considered before intervening on \(k \to y\).  
If \(k^\prime \to y\) is included in the intervention set \(\mathcal{S}_{\pi}^i\) before \(k \to y\), then the baseline for \(k \to y\) is conditioned on the effects of \(k^\prime \to y\). 
In DAG-SHAP, all valid permutations \(\pi \in \Pi\) are considered. For permutations where \(k^\prime \to y\) precedes \(k \to y\), the cooperative effect is explicitly reflected in the marginal contribution.
This summation ensures that the attribution value for \(k\) includes benefits from \(k^\prime \to y\) in all cases where \(k^\prime\) precedes \(k\). Therefore, DAG-SHAP satisfies the Externality property.
\end{proof}

\begin{proof}[Proof of Exogeneity]
  In DAG-SHAP, the independent contribution of feature \(k\) is derived by intervening on the edges \(e_i = (k, c_i, \vx_{k})\) where $e_i \in \mathcal{O}_k$ while respecting the causal order defined by the directed acyclic graph (DAG).  
   The attribution value \(\Psi(e_i)\) is calculated as:
   \[
   \Psi(e_i) = \sum_{\pi \in \Pi} \frac{1}{|\Pi|} \left[\mathbb{E}[f(\rvx) \mid \operatorname{do}(\re_{{\underline{\mathcal{S}}}_{\pi}^i} = e_{{\underline{\mathcal{S}}}_{\pi}^i})] - \mathbb{E}[f(\rvx) \mid \operatorname{do}(\re_{\mathcal{S}_{\pi}^i} = e_{\mathcal{S}_{\pi}^i})]\right],
   \]
where \({\underline{\mathcal{S}}}_{\pi}^i = \mathcal{S}_{\pi}^i \cup \{e_i\}\) includes all preceding edges in the permutation \(\pi\) and the current edge \(e_i\), \(\mathcal{S}_{\pi}^i\) includes only the preceding edges.
   The expectation \(\mathbb{E}[\cdot]\) measures the contribution of \(e_i\) under specific interventions, ensuring no confounding from subsequent edges.
   DAG-SHAP adheres to the topological order \(\pi \in \Pi\), ensuring that for any edge \(e_j \in \mathcal{O}_k\), the influences of all ancestor nodes of \(k\) (those connected via paths to \(p_k\)) are already fully propagated before \(e_i\) is intervened upon. This means \(
   \Psi(e_j)\)  reflects only the influence of  $k$.
   By summing over all valid permutations \(\pi \in \Pi\), the attribution value for \(e_j\) accounts for its independent contribution across all possible causal configurations. Nodes not causally connected to \(k\) do not influence \(\Psi(e_j)\) because they are excluded from the intervention sets \(\mathcal{S}_{\pi}^k\) and \({\underline{\mathcal{S}}}_{\pi}^k\). Ancestor nodes' effects are already fully propagated by the time \(e_j\) is intervened upon.
   Thus, DAG-SHAP satisfies the Exogeneity property.
\end{proof}

\section{Comparison of Edge Intervention and Asymmetric Sampling Node Intervention}\label{sec:comparasion_asy_Intervention}
For asymmetric sampling node intervention, we use the following example to explain why it may fail. Let $X_1 = \mathbf{x}_1$, where $\mathbf{x}_1$ is a random variable uniformly distributed on $[0,1]$, representing the exogenous influence of $X_1$; $X_2 = X_1 \cdot \mathbf{x}_2$, where $\mathbf{x}_2$ is another random variable uniformly distributed on $[0,1]$, representing the exogenous influence of $X_2$. The generation of $Y$ follows $Y = X_1 \cdot X_2$. In summary, $X_1$ directly influences $Y$ and indirectly influences $Y$ through $X_2$. $X_2$ influences $Y$ with its own exogenous influence $\mathbf{x}_2$ and transfers the indirect influence of $X_1$. We aim to attribute values to each feature of a specific explained input $x^* = [x_1^*, x_2^*] = [1, 1]$ with respect to the baseline [0, 0].

For asymmetric sampling node interventions, the only valid sample permutation is $(x_1^*, x_2^*)$. The marginal contributions of $x_1^*$ and $x_2^*$ in the permutation $(x_1^*, x_2^*)$ are shown in the Table~\ref{tab:asym_node_marginal}.

\begin{table}[htbp]
\centering
\caption{Asymmetric-sampling node interventions: marginal contributions under the only valid ordering $(x_1^*,x_2^*)$.}
\label{tab:asym_node_marginal}
\setlength{\tabcolsep}{5pt}
\renewcommand{\arraystretch}{1.15}
\begin{tabularx}{\columnwidth}{@{}C{0.30\columnwidth}Y@{}}
\toprule
\textbf{Step (ordering)} & \textbf{Marginal contribution (with calculation)} \\
\midrule
$x_1^*$ in $(x_1^*,x_2^*)$ &
$\mathbb{E}\!\left[\mathbf{Y}\mid \operatorname{do}\!\left(\mathbf{x}_{\{1\}}=x_{\{1\}}^*\right)\right]
\allowbreak-\allowbreak
\mathbb{E}\!\left[\mathbf{Y}\mid \operatorname{do}(\emptyset)\right]
\allowbreak=\allowbreak
\frac{1}{2}\allowbreak-\allowbreak 0\allowbreak=\allowbreak \frac{1}{2}$ \\
$x_2^*$ in $(x_1^*,x_2^*)$ &
$\mathbb{E}\!\left[\mathbf{Y}\mid \operatorname{do}\!\left(\mathbf{x}_{\{1,2\}}=x_{\{1,2\}}^*\right)\right]
\allowbreak-\allowbreak
\mathbb{E}\!\left[\mathbf{Y}\mid \operatorname{do}\!\left(\mathbf{x}_{\{1\}}=x_{\{1\}}^*\right)\right]
\allowbreak=\allowbreak
1\allowbreak-\allowbreak \frac{1}{2}\allowbreak=\allowbreak \frac{1}{2}$ \\
\bottomrule
\end{tabularx}
\end{table}

Thus, the attribution value assigned to $x_1^*$ by the asymmetric sampling node intervention is $1/2$, and the attribution value assigned to $x_2^*$ is also $1/2$.

For edge intervention, we denote the edge $X_1 \to X_2$ as $\mathbf{e}_1$, $X_1 \to Y$ as $\mathbf{e}_2$, and $X_2 \to Y$ as $\mathbf{e}_3$. We denote the edges of the instance $x^*$ as $e_1^*, e_2^*, e_3^*$. According to the definition of DAG-SHAP , there are three valid edge permutations: $(e_1^*, e_2^*, e_3^*)$, $(e_1^*, e_3^*, e_2^*)$, and $(e_2^*, e_1^*, e_3^*)$. The marginal contributions of each edge in these permutations are as shown in Table~\ref{tab:edge_intervention_marginals_asym}.

\begin{table}[htbp]
\centering
\caption{Edge interventions (DAG-SHAP): marginal contributions under all valid edge permutations.}
\label{tab:edge_intervention_marginals_asym}
\setlength{\tabcolsep}{5pt}
\renewcommand{\arraystretch}{1.15}
\begin{tabularx}{\columnwidth}{@{}C{0.30\columnwidth}Y@{}}
\toprule
\textbf{Edge and permutation} & \textbf{Marginal contribution (with calculation)} \\
\midrule
$e_1^*$ in $(e_1^*,e_2^*,e_3^*)$ &
$\mathbb{E}\!\left[\mathbf{Y}\mid \operatorname{do}\!\left(\mathbf{e}_{\{1\}}=e_1^*\right)\right]
\allowbreak-\allowbreak
\mathbb{E}\!\left[\mathbf{Y}\mid \operatorname{do}(\emptyset)\right]
\allowbreak=\allowbreak 0\allowbreak-\allowbreak 0\allowbreak=\allowbreak 0$ \\

$e_1^*$ in $(e_1^*,e_3^*,e_2^*)$ &
$\mathbb{E}\!\left[\mathbf{Y}\mid \operatorname{do}\!\left(\mathbf{e}_{\{1\}}=e_1^*\right)\right]
\allowbreak-\allowbreak
\mathbb{E}\!\left[\mathbf{Y}\mid \operatorname{do}(\emptyset)\right]
\allowbreak=\allowbreak 0\allowbreak-\allowbreak 0\allowbreak=\allowbreak 0$ \\

$e_1^*$ in $(e_2^*,e_1^*,e_3^*)$ &
$\mathbb{E}\!\left[\mathbf{Y}\mid \operatorname{do}\!\left(\mathbf{e}_{\{1,2\}}=e_{\{1,2\}}^*\right)\right]
\allowbreak-\allowbreak
\mathbb{E}\!\left[\mathbf{Y}\mid \operatorname{do}\!\left(\mathbf{e}_{\{2\}}=e_2^*\right)\right]
\allowbreak=\allowbreak \frac{1}{2}\allowbreak-\allowbreak 0\allowbreak=\allowbreak \frac{1}{2}$ \\

$e_2^*$ in $(e_1^*,e_2^*,e_3^*)$ &
$\mathbb{E}\!\left[\mathbf{Y}\mid \operatorname{do}\!\left(\mathbf{e}_{\{1,2\}}=e_{\{1,2\}}^*\right)\right]
\allowbreak-\allowbreak
\mathbb{E}\!\left[\mathbf{Y}\mid \operatorname{do}\!\left(\mathbf{e}_{\{1\}}=e_1^*\right)\right]
\allowbreak=\allowbreak \frac{1}{2}\allowbreak-\allowbreak 0\allowbreak=\allowbreak \frac{1}{2}$ \\

$e_2^*$ in $(e_1^*,e_3^*,e_2^*)$ &
$\mathbb{E}\!\left[\mathbf{Y}\mid \operatorname{do}\!\left(\mathbf{e}_{\{1,2,3\}}=e_{\{1,2,3\}}^*\right)\right]
\allowbreak-\allowbreak
\mathbb{E}\!\left[\mathbf{Y}\mid \operatorname{do}\!\left(\mathbf{e}_{\{1,3\}}=e_{\{1,3\}}^*\right)\right]
\allowbreak=\allowbreak 1\allowbreak-\allowbreak 0\allowbreak=\allowbreak 1$ \\

$e_2^*$ in $(e_2^*,e_1^*,e_3^*)$ &
$\mathbb{E}\!\left[\mathbf{Y}\mid \operatorname{do}\!\left(\mathbf{e}_{\{2\}}=e_2^*\right)\right]
\allowbreak-\allowbreak
\mathbb{E}\!\left[\mathbf{Y}\mid \operatorname{do}(\emptyset)\right]
\allowbreak=\allowbreak 0\allowbreak-\allowbreak 0\allowbreak=\allowbreak 0$ \\

$e_3^*$ in $(e_1^*,e_2^*,e_3^*)$ &
$\mathbb{E}\!\left[\mathbf{Y}\mid \operatorname{do}\!\left(\mathbf{e}_{\{1,2,3\}}=e_{\{1,2,3\}}^*\right)\right]
\allowbreak-\allowbreak
\mathbb{E}\!\left[\mathbf{Y}\mid \operatorname{do}\!\left(\mathbf{e}_{\{1,2\}}=e_{\{1,2\}}^*\right)\right]
\allowbreak=\allowbreak 1\allowbreak-\allowbreak \frac{1}{2}\allowbreak=\allowbreak \frac{1}{2}$ \\

$e_3^*$ in $(e_1^*,e_3^*,e_2^*)$ &
$\mathbb{E}\!\left[\mathbf{Y}\mid \operatorname{do}\!\left(\mathbf{e}_{\{1,3\}}=e_{\{1,3\}}^*\right)\right]
\allowbreak-\allowbreak
\mathbb{E}\!\left[\mathbf{Y}\mid \operatorname{do}\!\left(\mathbf{e}_{\{1\}}=e_1^*\right)\right]
\allowbreak=\allowbreak 0\allowbreak-\allowbreak 0\allowbreak=\allowbreak 0$ \\

$e_3^*$ in $(e_2^*,e_1^*,e_3^*)$ &
$\mathbb{E}\!\left[\mathbf{Y}\mid \operatorname{do}\!\left(\mathbf{e}_{\{1,2,3\}}=e_{\{1,2,3\}}^*\right)\right]
\allowbreak-\allowbreak
\mathbb{E}\!\left[\mathbf{Y}\mid \operatorname{do}\!\left(\mathbf{e}_{\{1,2\}}=e_{\{1,2\}}^*\right)\right]
\allowbreak=\allowbreak 1\allowbreak-\allowbreak \frac{1}{2}\allowbreak=\allowbreak \frac{1}{2}$ \\
\bottomrule
\end{tabularx}
\end{table}

Thus, the attribution value assigned to $x_1^*$ by DAG-SHAP is $(0+0+1/2)/3 + (1/2+1+0)/3 = 2/3$,and the attribution value assigned to $x_2^*$ is $(1/2+0+1/2)/3 = 1/3$. As $x_1^*=1$ directly influences $Y$ through $X_2$ and the interaction $Y = X_1 \cdot X_2$, it is evident that $x_1^*=1$ is more important than $x_2^*$. Therefore, the asymmetric sampling node intervention provides incorrect attribution, as it fails to account for the external contribution of $x_1^* = 1$.

\end{document}